\documentclass[10pt,fleqn]{article}

\usepackage[margin=1.5in]{geometry}
\usepackage{amssymb}
\usepackage{latexsym}
\usepackage{amsmath}
\usepackage{amsthm}
\usepackage{graphicx}
\usepackage{authblk}
\usepackage{hyperref}
\usepackage{tipa}
\usepackage{float}
\usepackage{textgreek}
\usepackage{booktabs}
\floatstyle{plaintop}
\restylefloat{table}

\usepackage[OT2,OT1]{fontenc}
\usepackage[russian,english]{babel}

\providecommand{\keywords}[1]
{
   \textbf{{Keywords}} #1
}

\title{A Graph Diffusion Algorithm for Lexical Similarity Evaluation}
\author[1]{Karol Mikula}
\author[1]{Mariana Sarkociov\'{a} Reme\v{s}\'{\i}kov\'{a}}
\affil[1]{Department of Mathematics and Constructive Geometry, Faculty of Civil Engineering, Slovak University of Technology, Bratislava, Slovakia}

\usepackage{fancyhdr}
\DeclareMathAlphabet{\pazocal}{OMS}{zplm}{m}{n}

\newcommand\mt{t\kern-0.035cm\char39\kern-0.03cm}
\newcommand\ml{l\kern-0.035cm\char39\kern-0.03cm}
\newcommand\md{d\kern-0.035cm\char39\kern-0.03cm}
\newcommand{\mL}{L\kern-0.14cm\char39}

\theoremstyle{definition}
\newtheorem{theorem}{Theorem}[section]

\newtheorem{prop}[theorem]{Proposition}
\newtheorem{rem}[theorem]{Remark}

\begin{document}


\date{}

\maketitle


\vspace{1cm}

\begin{abstract}
	In this paper, we present an algorithm for evaluating lexical similarity between a given language and several reference language clusters. As an input, we have a list of concepts and the corresponding translations in all considered languages. Moreover, each reference language is assigned to one of $c$ language clusters. For each of the concepts, the algorithm computes the distance between each pair of translations.  Based on these distances, it constructs a weighted directed graph, where every vertex represents a language. After, it solves a graph diffusion equation with a Dirichlet boundary condition, where the unknown is a map from the vertex set to $\mathbb{R}^c$. The resulting coordinates are values from the interval $[0,1]$ and they can be interpreted as probabilities of belonging to each of the clusters or as a lexical similarity distribution with respect to the reference clusters. The distances between translations are calculated using phonetic transcriptions and a modification of the Damerau-Levenshtein distance. The algorithm can be useful in analyzing relationships between languages spoken in multilingual territories with a lot of mutual influences. We demonstrate this by presenting a case study regarding various European languages.
\end{abstract}


\section{Introduction}

Lexical similarity is one of the key criteria in examining relationships between languages. Given two languages and a list of concepts, it measures the ratio of similar corresponding terms in that particular part of vocabulary. The resemblance can be measured using various characteristics of words and various metrics based on those characteristics.

In computational linguistics, there are several common approaches to lexical similarity evaluation. One of the basic properties of any word is its etymology and the etymology based lexical similarity has been used for decades in lexicostatistics. This discipline is based on the ideas of Morris Swadesh \cite{Sw1,Sw2} and typically uses his list of basic concepts, the so-called Swadesh list. This approach has been widely applied in classifying languages and supporting hypotheses regarding their origin, for example in the works of Dyen et al. \cite{Dyen}, Gray and Atkinson \cite{GA}, Kassian et al. \cite{Kass}, Petroni and Serva \cite{PS}.

Besides etymology, phonetic similarities are often explored and used to classify languages and dialects. The related algorithms measure distances between phonetic transcriptions of words or sentences, most typically using Levenshtein distance or some of its modifications, less commonly the Euclidean distance or some other metric. Some examples are classification of Irish Gaelic dialects by Kessler \cite{Kessler}, Norwegian dialects by Heeringa \cite{Heer}, Dutch dialects by Nerbonne et al. \cite{Nerbonne} and Indonesian languages by Nasution et al. \cite{Nasu}. Another example is prediction of mutual intelligibility of Chinese dialects by Tang and Van Heuven \cite{Tang} or exploring the similarity of a Native American language with several other well-known languages in the work of Daniels et al. \cite{Daniels}.

To mention some other approaches, a character-based classification of Japanese dialects and Slavic languages can be found in the work of Sato and Heffernan \cite{Sato}. A completely different procedure is proposed in the paper of Lin \cite{Lin}, where the similarity of words is evaluated based on the context in which they typically appear. 

The main motivation for our work was exploring the similarity in specific subsets of the lexis and between languages that are spoken in multilingual territories. Such languages often contain parts of vocabulary that are in discordance with their standard overall classification. Identifying these parts can help understand the language and its evolution and possibly even provide an insight in the way of life of their speakers and their relationships with other ethnic groups.

The algorithm that we present is based on solving the stationary diffusion equation with a boundary condition on a weighted directed graph. The vertices of the graph represent the languages included in the experiment, which are of three different types. The first group consists of languages with known classification and these are divided into several clusters (language branches). The second group contains languages that we want to classify because we assume a notable foreign influence or resemblance with languages from different branches and we want to analyze it. Finally, for each of the involved language clusters, we add one hypothetical language which is used to set the boundary condition of the model. Given a list of concepts, we solve our model equation for each concept individually. The edge weights are re-computed in each step based on the phonetic distance between the corresponding words. As the solution, we obtain a Euclidean graph that reflects the similarities between all corresponding words from all considered non-hypothetical languages. The overall lexical similarity between languages is then evaluated based on all individual results.

Our procedure is inspired by the work of Mikula et al. \cite{Mikula} that introduces a clustering and classification algorithm based on non-stationary non-linear diffusion on graphs, with no boundary condition. The algorithm was successfully applied in classification of Natura 2000 habitats using satellite images. The vertices of the involved graph correspond to the samples measured by a satellite. Based on similarity of the measured features, the coordinates of the samples in the feature space are evolved by diffusion, either forward or backward. As a result, new measurements could be classified and inserted in the already existing clusters.

Contrarily to the satellite imaging measurements, in case of words and languages, it is not straightforwardly clear what the feature space should be. In our case, each experiment includes languages from several distinct clusters (language families or branches). Therefore, we set the features of a word to be its probabilities of belonging to each of the clusters. Or, in other words, the values of the features represent the similarity distribution of each classified word with respect to the reference language clusters. The overall output of our algorithm is also understood in this way -- rather than providing the lexical similarity in the form of the ratio of similar words, we obtain the lexical similarity distribution for the given language and the considered language clusters.

There are many algorithms for clustering and classification operating with probabilities and it is also not uncommon to use the graph Laplacian for clustering on graphs, see e.g. \cite{Bertozzi,Ding,Yu}. However, to the best of our knowledge, the procedure described in this paper has not been previously used. Its advantages are its simple and elegant formulation and a rather non-complicated implementation. Equally importantly, our clustering and classification algorithm only needs a low-key input: any type of data and an arbitrary way of measuring the distances between them. The data need not be expressed in any feature coordinate system, they are simply elements of a data set.

Our paper consists of several parts. First, we formulate the basic mathematical model. After, since our experiments use the phonetic transcription of words, we describe the semi-metric used for measuring the distance between such transcriptions. Finally, we present results of experiments concerning several interesting European languages. The concept set that we used consists of traditional hand and agricultural tools and other related concepts. 


\section{Mathematical model}\label{SecMathematicalModel}

Before we start explaining our model, let us first clear out some terminology. By a {\em translation} of a concept, we mean a list of words (synonyms) that are used for the given concept in a given language. Given the concept list, our algorithm takes the concepts one-by-one and evaluates the relationships between its translations in all given languages. 

What we are going to describe now is the procedure performed for an individual concept. Focusing on a single concept, each translation has the role of representing the whole language to which it belongs. Keeping in mind this correspondence, we will often explain our thoughts in terms of languages rather than translations. In certain contexts, it allows to grasp the principles of our method in a more to-the-point way. 

As we have already mentioned, our algorithm considers three groups of languages. The first group is composed of languages that are considered to be already classified (belonging to some language cluster). In the rest of the paper, we will call them {\em reference languages}. The second group contains the languages that we would like to classify and they will be referred to as {\em classified languages}. In the third group, we have the {\em hypothetical languages} representing each of the clusters. For a given concept, the hypothetical languages can be thought of as languages that use all terms found in their corresponding clusters as synonyms. 

Now, let us establish some notation that our model uses. In what follows,
\begin{itemize}
	\item{$n$ is the number of reference languages,}
	\item{$m$ is the number of classified languages,}
	\item{$c$ is the number of reference clusters that we are working with.}
\end{itemize}
Further, in our description, we will use several graphs where each vertex represents one of the involved languages. In addition to the above notation, we will use three related functions: $\nu$, $\lambda$ and $head$.
\begin{itemize}
	\item{If a graph vertex $v$ represents a reference language, then $\nu(v)$ is the number of languages in the corresponding reference cluster. If $v$ represents a hypothetical language, then $\nu(v)=\nu_0$, where ${\nu}_0$ is a positive constant.}
	\item{For any graph edge $e$, $\lambda(e)$ is the distance between the translations represented by the endpoints of $e$ (measured in any chosen way).}
	\item{If $e$ is a directed edge, then $head(e)$ is its head.}
\end{itemize}

We will build up our model in three steps. In the first step, let us assume that all reference clusters have an equal number of elements and let us consider only the reference languages and the hypothetical languages. Let us have an undirected graph $G_0=(V_0,E_0)$, where the vertices $v_1,\dots,v_n$ represent the reference languages and the vertices $h_1,\dots,h_c$ correspond to the hypothetical languages and are designated as the boundary of the graph (in the sense of Friedman and Tillich \cite{FT}, i.e. they are vertices where a boundary condition will be prescribed). The set $E_0$ contains edges connecting any $v_i$ and $v_j$, $i=1,\dots,n$, $j=1,\dots,n$, $i\neq j$. In addition, each vertex $v_i$, $i=1,\dots,n$, is connected with the vertex representing its corresponding hypothetical language. 

The graph $G_0$ together with the distance measuring function $\lambda$ are our starting point: $G_0$ says which languages have a chance to influence each other and $\lambda$ encodes the similarity between all pairs of translations. What we want to obtain is a Euclidean graph that reflects the similarities given by $\lambda$, i.e. similar translations should form clusters or at least approach each other. As we have mentioned in the introduction, the Euclidean coordinates of a language represent its probabilities of belonging to each of the $c$ reference clusters. Thus, our feature space is $\mathbb{R}^c$ and we are looking for a Euclidean graph $\mathcal{G}_0=(G_0,\varphi)$, where $\varphi\colon V_0\rightarrow\mathbb{R}^c$. A natural and elegant way of clustering is non-uniform diffusion which smooths out the differences between data that allow an intense diffusion flux between them. Let us therefore consider a diffusion intensity function $g\colon\mathbb{R}\rightarrow\mathbb{R}$, which is a positive decreasing function. Based on $g$, we make $G_0$ a weighted graph where the weight of and edge $e$ is $w_0(e)=g(\lambda(e))$. Then, we find $\varphi$ as the solution of the equation
\begin{equation}
	\Delta_{w_0}\varphi=0, \quad \varphi(h_i)=\iota_i, \, i=1,\dots,c,
	\label{EqLaplace}
\end{equation}
where $\iota_i$ are the standard basis vectors of $\mathbb{R}^c$. The operator $\Delta_{w_0}$ is the weighted graph Laplacian whose value at the vertex $v$ is
\begin{equation}
	(\Delta_{w_0} \varphi)(v)  = \sum\limits_{e:u\sim v} w_0(e)(\varphi(u)-\varphi(v)),
	\label{EqGraphLaplace}
\end{equation}
where $e:u\sim v$ means that the edge $e$ connects $u$ and $v$.

Let us now explain how exactly the model (\ref{EqLaplace})--(\ref{EqGraphLaplace}) works. The points $\varphi(v_i)$, $i=1,\dots,n$, are mutually attracted according to the diffusion intensity $g$, which depends on the similarity between the corresponding translations. However, even though the similarities differ, just setting $\Delta_{w_0}\varphi$ equal to $0$ would cause all $\varphi(v_i)$ collapsing into a single point. To prevent that, Mikula et al. \cite{Mikula} suggested to use backward diffusion causing repulsion between clusters. In our case, we consider the hypothetical languages that are natural representatives of the clusters and we set the Dirichlet boundary condition using them. Since the hypothetical languages are assumed to contain all terms from their corresponding clusters, their distance from any language in that cluster is set to zero (see Section \ref{SecDealingWithSynonyms} for details). This implies that the corresponding diffusion coefficient has the maximum possible value and all languages are strongly attracted to their hypothetical language. This reflects the a-priori information about where they belong. However, they will be more or less attracted also to all other reference languages. If a translation does not find a similar counterpart anywhere else, it will remain close to its hypothetical language and its similarity distribution will be close to a basis vector. If there is a similar term in another cluster, it will be shifted towards it and the length of the shift will depend on the similarity rate. Also, the larger the number of similar terms in the other cluster, the more prominent the shift. 

Of course, many times, the assumption of an equal number of elements in all reference clusters does not meet reality -- there are language branches that contain a lot of languages and others that contain just one or two. In this case, the basic model (\ref{EqLaplace})--(\ref{EqGraphLaplace}) might lead to unrealistic results. To see this, let us consider a simple example with only two reference clusters. The first cluster contains $n_1$ languages represented by $v_1,\dots,v_{n_1}$ and the second cluster consists of $n_2=n-n_1$ languages represented by $v_{n_1+1},\dots,v_n$. To illustrate the essence of the problem, let us assume that all translations in both clusters are identical. Of course, this not something that we normally see in practice, however, we do encounter similar situations when the translations across clusters do not differ a lot. Under the assumption of identity, $w_0(e)$ has the same value for every $e\in E_0$. Also, all languages within an individual cluster are equally similar to each other, to their hypothetical language and to all languages from the other cluster. Thus, $\varphi$ will map all of them onto the same point in $\mathbb{R}^2$. Let us use the notation $\varphi_1=\varphi(v_1)=\dots=\varphi(v_{n_1})$ and $\varphi_2=\varphi(v_{n_1+1})=\dots=\varphi(v_{n})$. The equations (\ref{EqLaplace}) and (\ref{EqGraphLaplace}) then lead to
\begin{eqnarray*}
	n_2(\varphi_2-\varphi_1)+(\iota_1-\varphi_1) & = & 0, \\
	n_1(\varphi_1-\varphi_2)+(\iota_2-\varphi_2) & = & 0.
\end{eqnarray*}
The solution of this system is
\[ \varphi_1  =  \frac{n_1+1}{n_1+n_2+1}\iota_1+\frac{n_2}{n_1+n_2+1}\iota_2, \quad \varphi_2  =  \frac{n_1}{n_1+n_2+1}\iota_1+\frac{n_2+1}{n_1+n_2+1}\iota_2.\]
For a fixed $n_1$, we can see that the bigger the value of $n_2$, the lower the similarity between any language from the first cluster and its own cluster. In fact, if $n_1=1$, then as soon as $n_2=3$, the language is already more similar to the second cluster ($\varphi_1=0.4\iota_1+0.6\iota_2$). For $n_2\rightarrow\infty$, we have $\varphi_1\rightarrow\iota_2$ and also $\varphi_2\rightarrow\iota_2$, which means that the limit of the similarity with the first cluster is zero. This is not what we would expect as a realistic result.

In the second step, we adjust the basic model to compensate for this unwanted side effect. We do this by considering an directed graph $G_1=(V_0,E_1)$ with the same vertex set as $G_0$ but with directed edges. For each $v_i$ and $v_j$, $i=1,\dots,n$, $j=1,\dots,n$, $i\neq j$, $E_1$ contains both edges $v_i\rightarrow v_j$ and $v_j\rightarrow v_i$. In addition, for each $i=1,\dots,n$, we have an edge directed from $v_i$ to the boundary vertex representing the corresponding hypothetical language. The edge weights are now, in general, asymmetric and defined as
\begin{equation}
	w(e)=\frac{w_0(e)}{\nu(head(e))}.
	\label{Eqw}
\end{equation}
Note that for all edges whose heads are hypothetical languages, we have the same weight $g(0)/\nu_0$. With this setting, we now have the model with an adjusted operator
\begin{equation}
	\Delta_{w}\varphi=0, \quad \varphi(h_i)=\iota_i, \, i=1,\dots,c,
	\label{EqLaplaceFinal}
\end{equation}
\begin{equation}
	(\Delta_{w} \varphi)(v)  = \sum\limits_{e:v\rightarrow u} w(e)(\varphi(u)-\varphi(v)),
	\label{EqGraphLaplaceFinal}
\end{equation}
where $e\colon v\rightarrow u$ means that the edge $e$ is directed from $v$ to $u$. 

Let us verify that adding this asymmetry resolves our trouble with clusters of a different size. We use the same setting as before, with two clusters containing $n_1$ and $n_2$ identical translations. From (\ref{EqLaplaceFinal})--(\ref{EqGraphLaplaceFinal}), we now get
\begin{eqnarray*}
	n_2\frac{1}{n_2}(\varphi_2-\varphi_1)+\frac{1}{\nu_0}(\iota_1-\varphi_1) & = & 0, \\
	n_1\frac{1}{n_1}(\varphi_1-\varphi_2)+\frac{1}{\nu_0}(\iota_2-\varphi_2) & = & 0.
\end{eqnarray*}
This system has the solution
\begin{equation}
	\varphi_1  =  \frac{\nu_0+1}{2\nu_0+1}\iota_1+\frac{\nu_0}{2\nu_0+1}\iota_2, \quad \varphi_2  =  \frac{\nu_0}{2\nu_0+1}\iota_1+\frac{\nu_0+1}{2\nu_0+1}\iota_2.
	\label{EqSolutionNu0}
\end{equation}
As we can see, this solution does not depend on the cluster sizes $n_1$ and $n_2$ and it is only determined by the value of $\nu_0$. Moreover, we have
\[ \Vert\varphi_1-\iota_1\Vert<\Vert\varphi_1-\iota_2\Vert, \quad \Vert\varphi_2-\iota_2\Vert<\Vert\varphi_2-\iota_1\Vert \]
and
\[\frac{\varphi_1+\varphi_2}{2}=\frac{\iota_1+\iota_2}{2}. \]
This means that the clusters will be placed symmetrically around the arithmetic mean of $\iota_1$ and $\iota_2$ and any language will be always more similar to its own cluster than to the other cluster. For $\nu_0\rightarrow\infty$, the similarity will converge to $\frac{1}{2}$. An example with real world data validating this approach will be shown later in Section \ref{SecCaseStudy} (Figure \ref{FigClustersSize2}).

At last, to construct our final graph $G=(V,E)$, we include the classified languages. These languages are attracted to the reference languages according to their similarity, but they themselves do not attract any other language. Thus, $V$ is constructed from $V_0$ by adding $m$ vertices $l_1,\dots,l_m$ and $E$ is created by adding $mn$ edges directed from $l_i$ to $v_j$, $i=1,\dots,m$, $j=1,\dots,n$, to the edge set $E_1$. The edge weights in the whole graph are set according to (\ref{Eqw}) and after, we apply the model (\ref{EqLaplaceFinal})--(\ref{EqGraphLaplaceFinal}). The final setting is illustrated in Figure \ref{FigGraphs}.

\begin{figure}[h]
	\centering
	\includegraphics[width=0.65\textwidth]{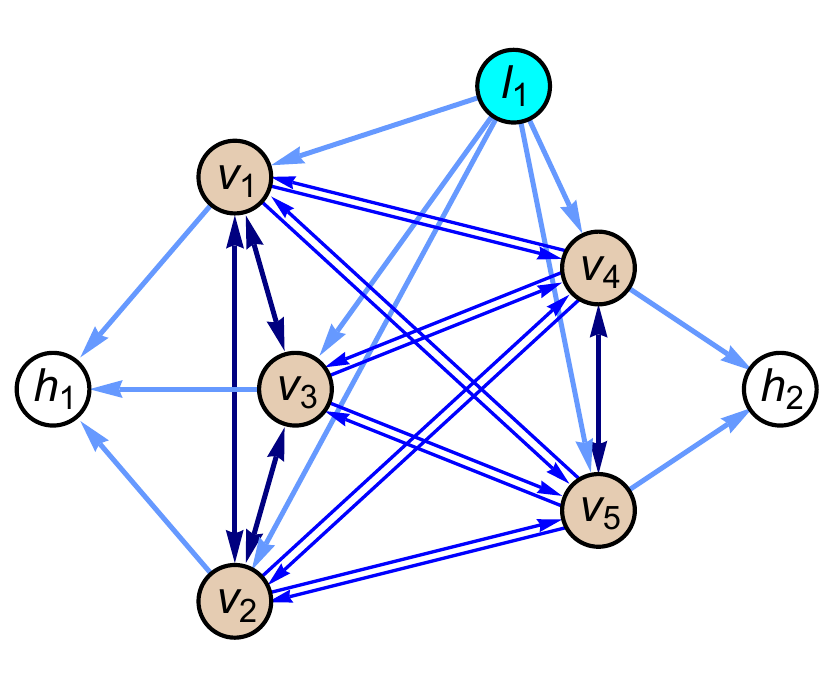}
	\caption{The directed graph $G$ used in our model. Here, we have 5 reference languages $v_1,\dots, v_5$, one classified language $l_1$ and two hypothetical languages: $h_1$ for the cluster $\{v_1,v_2,v_3\}$ and $h_2$ for the cluster $\{v_4,v_5\}$. The two-sided arrows represent pairs of directed edges, where the diffusion coefficients are equal in both directions. On the pairs of edges connecting the two clusters, we have an asymmetric diffusion, since the clusters are of different sizes. There is a one-way diffusion between the classified language and all reference languages and also between each reference language and its corresponding hypothetical language.}
	\label{FigGraphs}
\end{figure}

Now, let
\[ \varphi_i=\left\{ 
	\begin{array}{ll}
		\varphi(v_i), \quad & i=1,\dots n, \vspace{0.1cm} \\
		\varphi(l_{i-n}), \quad & i=n+1,\dots,n+m, \vspace{0.1cm} \\
		\varphi(h_{i-n-m})=\iota_{i-n-m}, \quad & i=n+m+1,\dots, n+m+c.
	\end{array}
	\right.
\] 
Further, let $w_{ij}$ be the diffusion coefficient (weight) corresponding to the $j$-th outgoing edge of $v_i$ and let $\varphi_{i_j}$ be the value of $\varphi$ in the head of that edge. The model (\ref{EqLaplaceFinal})--(\ref{EqGraphLaplaceFinal}) leads to a linear system with unknowns $\varphi_i$, $i=1,\dots,n+m$. From the construction of our final graph follows that any non-boundary vertex has $n$ outgoing edges. Thus, for the $i$-th unknown, we have
\begin{equation}
	-\left(\sum\limits_{j=1}^n w_{ij}\right)\varphi_i +\sum\limits_{j=1}^n w_{ij}\varphi_{i_j}=0.
	\label{EqLinSystem}
\end{equation}
Note that for $i=1,\dots,n$, one of the values $\varphi_{i_j}$ is not an unknown, but it represents the boundary condition and contributes to the right hand side of the system. Also, since $\varphi_i\in \mathbb{R}^c$, the equation (\ref{EqLinSystem}) actually represents $c$ independent linear systems with the same matrix but with different right hand sides.

After solving the system (\ref{EqLinSystem}) for each of the given concepts, we obtain a Euclidean graph in $\mathbb{R}^{pc}$, where $p$ is the number of concepts. This graph reflects the lexical similarity between all languages involved in our experiment and provides a global image of the relationships between them. 

The last thing that remains to show is that the solution of the system (\ref{EqLinSystem}) actually allows an interpretation in terms of probabilities or similarity distributions.

\begin{prop}
	Let $(\varphi_1,\dots,\varphi_{n+m})$ be the solution of the system (\ref{EqLinSystem}) and let $\varphi_i^k$ be the $k$-th component of $\varphi_i$. Then we have
	\begin{enumerate}
		\item{$\varphi_i^k\in [0,1]$,}
		\item{$\sum\limits_{j=1}^c\varphi_i^k=1$.}
	\end{enumerate}
\end{prop}
\begin{proof}
	Since the reference languages are only influenced by each other and the hypothetical languages, we will discuss them separately and start by proving that any $\varphi_i$, $i=1,\dots,n$, is a convex combination of $\iota_1,\dots,\iota_{c}$. Let us assume that this is not true. Then, for some $p\in\{1,\dots,n\}$, the point $\varphi_{p}$ must be a vertex of the convex hull $\mathcal{C}$ of all $\varphi_i$, $i=1,\dots,n$, and $\iota_i$, $i=1,\dots,c$. Let us consider any hyperplane $\Pi$ in $\mathbb{R}^c$ such that $\Pi\cap\mathcal{C}=\{\varphi_{p}\}$ and let $n_{\Pi}$ be a unit normal of $\Pi$. Then (\ref{EqLinSystem}) implies
	\begin{equation}
		\sum\limits_{j=1}^n w_{pj}(\varphi_{p_j}-\varphi_p)\cdot n_{\Pi}=0 .
		\label{EqSumVarphi}
	\end{equation}
	From the definition of $\Pi$ follows that all $\varphi_i$, $i=1,\dots,n$, lie in one of the closed half spaces bounded by $\Pi$. Thus, since all $w_{pj}$, $j=1,\dots,n$, $j\neq p$, are positive, all summands in (\ref{EqSumVarphi}) have the same sign. Further, by assumption, $\varphi_p\neq\iota_{i}$, $i=1,\dots,c$, which means that at least one of the summands is non-zero. This implies that the sum in (\ref{EqSumVarphi}) cannot be zero and we have come to a contradiction.
	
	The above implies that for any $i=1,\dots,n$ 
	\[ \varphi_i=\sum\limits_{j=1}^c a_{ij}\iota_j, \]
	where $\sum\limits_{j=1}^c a_{ij}=1$. Since $\iota_j^k$ is either 0 or 1, we have $\varphi_i^k\in[0,1]$. Further,
	\[ \sum\limits_{k=1}^c\varphi_i^k=\sum\limits_{k=1}^c\sum\limits_{j=1}^c a_{ij}\iota_j^k=\sum\limits_{j=1}^c a_{ij}\sum\limits_{k=1}^c\iota_j^k=\sum\limits_{j=1}^c a_{ij}=1.\]
	
	Finally, for $i=n+1,\dots,n+m$, the equality (\ref{EqLinSystem}) implies that $\varphi_i$ is a convex combination of $\varphi_j$, $j=1,\dots,n$. As such, it is also a convex combination of $\iota_j$, $j=1,\dots,c$, and, similarly as above, we can confirm the claims of our proposition.
\end{proof}

\begin{rem}
Using the graph calculus of Friedman and Tillich \cite{FT}, the equation (\ref{EqLaplace}) can be written as an actual differential equation (Laplace's equation) representing stationary diffusion on a weighted graph. For the oriented graph, this calculus does not apply anymore and our model has to be understood as a discrete and, moreover, non-symmetric analogue of diffusion. However, just as the classical diffusion, this type of diffusion process can be also found in natural phenomena, for example in chemical reaction networks \cite{Muller}. 
\end{rem}

\subsection{The diffusion intensity function}

To complete the model, it remains to specify the diffusion intensity function $g$ used above. In our case, we set
\begin{equation}
	g(x)=\frac{1}{1+(e^{Kx}-1)^2} 
	\label{EqFunctionf}
\end{equation}
where $K$ is a positive constant. This form of $g$ was chosen for two reasons. First, it is a decreasing function converging to zero, which is necessary, if we want similar translations to be attracted more than unrelated or less similar translations. Second, in order to choose from among all functions with this property, we inspected the ratio $g(x)/g(x+1)$. We would like this ratio not to converge to 1, which was the case for several other standard functions that we tested, since we want a higher distance between words to always yield a sufficiently decreased diffusion coefficient. Indeed, we have
\[ \lim\limits_{x\rightarrow\infty}\frac{g(x)}{g(x+1)}=e^{2K},\]
which means that we made an appropriate choice for our situation.


\section{Measuring the distance between translations}

\subsection{Distance between individual words}

As we could see in the previous section, the mathematical model can employ any function that measures the distance between two translations or two individual words. The method for individual words that we used measures the distance between phonetic transcriptions of the words and it is based on the Damerau-Levenshtein (DL) distance \cite{Dam} and basic properties of phonemes.  

The Levenshtein distance \cite{Lev} and its variants are a classic in the field of computational linguistics. In its basic form, it measures the minimum number of edits that transform the first string into the second or vice versa. The elementary edits are insertion, deletion and substitution. In addition to these three operations, the DL modification considers transposition as an elementary edit as well. Since transpositions are often observed when comparing equivalent terms from two different languages, we chose to use this variant. 

In its basic setting, the DL distance assigns an equal (unit) weight to all elementary edits. This can create unrealistic results especially when it comes to phoneme substitution. For example, the edit distance between the English word `red' \textipa{[\*rEd]} and the Danish word `r\o d' \textipa{[K\oe D]} is 3, since all three characters in their phonetic transcriptions are different, even though the words are cognates and obviously similar. The same edit distance is obtained between `red' and the completely unrelated Welsh equivalent `coch' \textipa{[ko:X]} (taking `\textipa{o:}' as one character). We address this issue by assigning substitution a weight based on the distance between the substituted phonemes. The rest of the elementary edits keep their original unit weight. 

Let $\delta$ denote the function of two string variables obtained by the above described modification of the DL distance. Contrarily to the original DL distance, $\delta$ need not be a metric on the set of strings. The reason is that the triangle inequality is not satisfied, if the substitution weight is allowed to be greater than 2. However, it is always a semi-metric, which is a common tool for evaluating similarity between data of various types (for examples, see e.g. \cite{Conci} or \cite{James}).

As another modification, we separated consonants and vowels and evaluated $\delta$ for the consonant and vowel substring separately. Consonants and vowels were found to have somewhat different roles in a language  and consonants are usually the ones carrying more lexical information \cite{Nespor}. Also, they tend to be more stable as a language undergoes sound changes and vowels are typically more prone to changing. As a typical example, let us take the English word `many' \textipa{[mEni]}. Some of its Germanic cognates are Scots `mony' \textipa{[m6nI]}, Danish `mange' \textipa{[mAN@]} or Swedish `m\aa nga' \textipa{[mONa]}. As we can see, none of the phonetic transcriptions has a matching vowel with any other one, while there is only a small variation in consonants. This is a common observation across probably most language branches, at least in the region that we were dealing with in our experiments. So, if we have two strings $p$ and $q$, we first form their consonant and vowel substrings $p_{con}$, $p_{vow}$, $q_{con}$, $q_{vow}$ and then we compose the distance measuring function $d_0$ as
\begin{equation}
	d_0(p,q)=w_{con}\delta(p_{con},q_{con})+w_{vow}\delta(p_{vow},q_{vow}). 
	\label{Eqd0}
\end{equation}
The coefficients $w_{con}$ and $w_{vow}$ are the consonant and vowel weights and we will speak about setting their values later.

We have found a similar approach to ours in the interesting and elaborated PhD. thesis of Heeringa \cite{Heer}. Here, the author does not consider transpositions and assigns different weights to all insertion, deletion and substitution. He also allows certain consonant-vowel substitutions.

\subsection{Dealing with synonyms}\label{SecDealingWithSynonyms}

When trying to properly classify or cluster languages, synonyms have to be taken into account. Let us have two languages and a given concept, for which the first language has a $k$-element translation (vector of synonyms) $p=\{p_1,\dots,p_k\}$ and the second one has an $l$-element translation $q=\{q_1,\dots,q_l\}$. The overall distance $d$ between $p$ and $q$ is evaluated according to the type of languages that we are dealing with.
\begin{itemize}
	\item{If both languages are reference languages, we compute $d$ as
	\begin{equation}
		d(p,q)=\frac{\sum\limits_{i=1}^k \min\limits_{j=1\dots l} d_0(p_i,q_j)+\sum\limits_{j=1}^l \min\limits_{i=1\dots k} d_0(p_i,q_j)}{k+l},
		\label{EqDistG}
	\end{equation}
	where $d_0$ is computed according to (\ref{Eqd0}).  }
	\item{If only one of the languages is a reference language, we set
	\begin{equation}
		d(p,q)=\min\limits_{i=1\dots k, j=1\dots l} d_0(p_i,q_j).
		\label{EqDistBarG}
	\end{equation}}
	This includes the situation when one of the languages is hypothetical. According to our interpretation of a hypothetical language, in that case we get $d(p,q)=0$.
\end{itemize}

The reason for using two different approaches is the following. For a classified language, we just need to evaluate its lexical similarity with the reference languages. This means that it is sufficient to find a term in each language that resembles its own term(s) the most. A hypothetical language is meant to strongly attract each language from the corresponding cluster and the attraction should be of the same intensity across all clusters. This naturally leads to setting $d$ to zero. In case of two reference languages, the situation is not so simple, since their evaluated lexical similarity affects the final classification of the classified languages. Now, let us assume that the first reference language has two synonyms $p_1$ and $p_2$ for a given concept. The second reference language from another cluster has an only term $q=p_1$. If they are the only reference languages in the experiment and if we use the distance (\ref{EqDistBarG}), they will end up at the same position in the resulting Euclidean graph. Let us further suppose that the classified language has a single term $\bar{q}=p_2$. For that reason, it will end up at the exact same position as the two reference languages, while it has nothing in common with the second reference language. This shows that taking a simple minimum (\ref{EqDistBarG}) could lead to a strongly unrealistic situation. The formula \ref{EqDistG} softens this effect and expresses that the reference languages are similar, but  not completely identical. Experiments have also shown that the results obtained in this way are more realistic.

\subsection{The phoneme distances}

In order to obtain a suitable substitution weight, we need to measure the distance between phonemes. Since we treat consonants and vowels separately, we will describe both consonant and vowel distances. 

Basically, there are two main approaches for evaluating the phoneme distance: the acoustic method and the feature method. The acoustic method is based on analyzing spectrograms corresponding to phonemes (first presented by Potter \cite{Potter}) and using different techniques to measure the distance between them (see e.g. \cite{Heer}, \cite{Vak}). The feature method uses the information about how a specific sound is produced in the human airway system. For our purposes, we chose this approach. 

The feature space used in the feature method is usually quite high-dimensional. For example, Rubehn et al. \cite{Rubehn} use a 39-dimensional space of ternary features. Hoppenbrouwers and Hoppenbrouwers \cite{Hopp} use 11 features for consonants and 10 features for vowels, while Heeringa \cite{Heer} considers 12 consonant features and 15 vowel features. A different 12-dimensional consonant feature space is used by Vakulenko \cite{Vak}. A high amount of features is usually useful in dialectology where the phoneme differences can be subtle and classification could be difficult without a fine resolution. 

In our work, we do not try to classify dialects but languages, so we can afford omitting some features that do not have a significant effect on the classification. We also did not include some features, for example the pulmonic and the click consonant feature, that do not play any role in the languages that we used. If needed for experiments with other languages, they can be easily added. Our final selection includes 6 consonant features and 5 vowel features. Let us now describe them more in detail.

First, we will discuss the consonants. To assemble our set of features and assign coordinates to each consonant, we used the extended consonant chart by Esling \cite{Esling}. However, we made some adjustments that will be explained below. The features that we included are:
\begin{enumerate}
	\item{{\em Articulation zone}. This is an integer value ranging from 0 to 11, corresponding to the 12 articulation zones: bilabial, labio-dental, linguo-labial, dental, alveolar, post-alveolar, retroflex, palatal, velar, uvular, pharyngeal, glottal.}
	\item{{\em Air flow type}. This parameter ranges from 0 to 9 and expresses the character of the air flow when pronouncing the consonant. Its values are: nasal (0), plosive (4), affricate (5), fricative (6), approximant (7), tap/flap (8) and trill (9).}
	\item{{\em Voicing}. The consonant can be either voiced or voiceless, so this value can be either 0 or 1.}
	\item{{\em Lateral feature}. Also 0 or 1, expressing whether the air flow is central or lateral.}
	\item{{\em Sibilant feature}. Again, 0 or 1, depending on whether the consonant is sibilant or not.}
	\item{{\em Coarticulated feature.} Coarticulated pronunciation employs two articulation zones. Thus, the first coordinate of a coarticulated consonant is set to the arithmetic mean of the coordinates representing the two articulation zones. Moreover, the value of the coarticulated feature is set to 1 (instead of 0 representing simple consonants).}  
\end{enumerate}

\begin{figure}[h]
	\centering
	\includegraphics[width=0.95\textwidth]{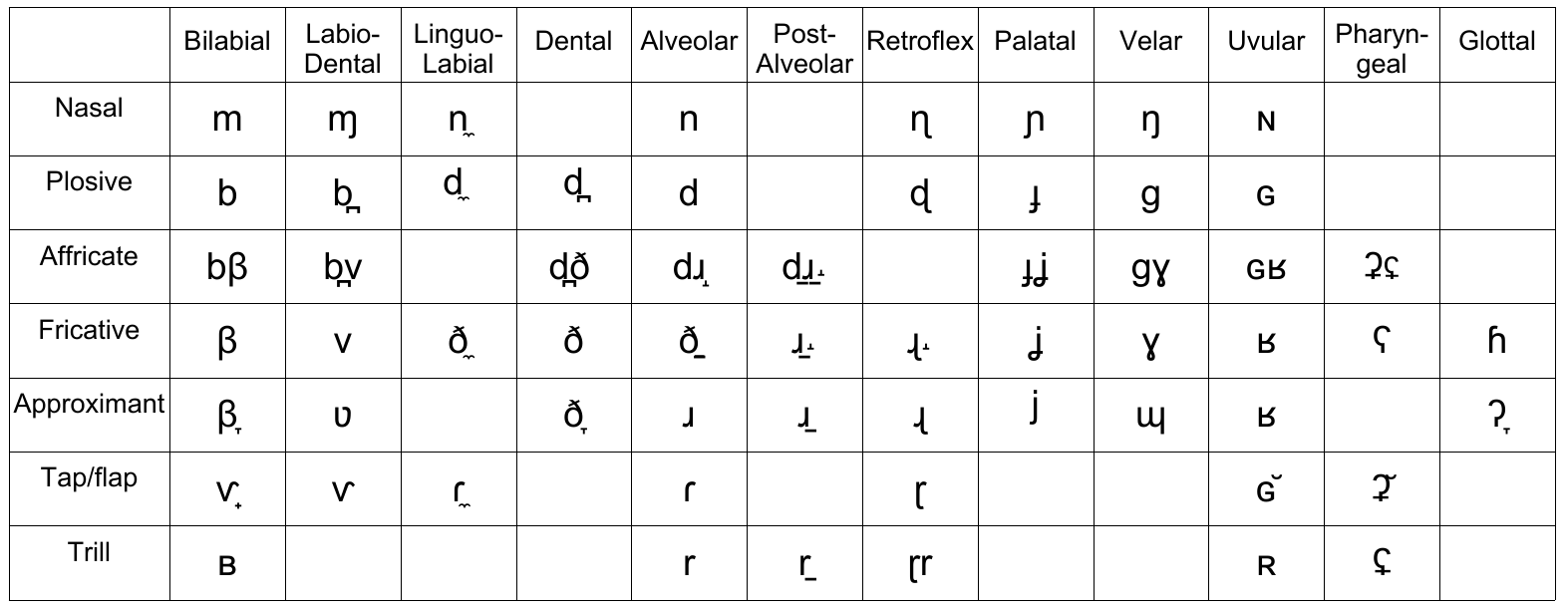} 
	\caption{Voiced, non-lateral, non-sibilant and non-coarticulated consonants represented in our model.}
	\label{FigConsonants}
\end{figure}

Using this six-dimensional coordinate system, we represented $168$ consonants. An illustration showing a 2D slice is provided in Figure \ref{FigConsonants}. Compared to Esling \cite{Esling}, the labio-dental and dento-labial consonants are not distinguished. Further, what needs to be cleared out the most, are the coordinates that we assigned to the various types of air flow. Esling and the official IPA chart \cite{IPA} list these types in a different order. To provide a reasoning behind our setting, let us first explain what the air flow types mean. If a consonant is nasal, it means the air flows out through the nose. For a plosive, the flow is stopped on the way out. A fricative is obtained by narrowing the air channel, which creates a turbulent flow. An affricate is a sound that starts as a plosive and releases as a fricative. An approximant is formed similarly to a fricative, but the air channel is not narrowed enough to create turbulence. For a tap/flap, there is a brief contact between the articulators. A trill is produced when the contact is longer and the articulator vibrates. 

This description intuitively suggests some possible orderings of the air flow types, but to obtain a more decisive conclusion, we employed another criterion. What the coordinates should actually do is to yield smaller distances for consonants that are frequently interchanged due to sound change and larger values for consonants that are improbable to substitute one another. It is not easy to find a relevant and transparent database of sound changes, however, there is a simple table of basic sound correspondences between Proto-Indo-European (PIE) and major Indo-European language groups by Mallory and Adams \cite[p. 464]{MA}. Since we worked mostly with Indo-European languages, we used the information from this table. According to it, the PIE nasal consonants did not undergo any systemic sound change and also no other PIE consonants were substituted by nasals. This fact and the particular character of the nasal air flow made us place nasals further from the other consonants. The order of plosives, affricates, fricatives and approximants comes naturally from their definition and there are also various corresponding sound changes in this group of consonants. As for the trills, the table claims only trill-approximant subtitutions (*l$\rightarrow$r, *r$\rightarrow$l), so we put them at the end of the list. Taps/flaps -- a softer version of trills -- were placed between approximants and trills.

For vowels, we used the classical IPA trapezoid \cite{IPA} (Figure \ref{FigVowels}). Since we set the distance between the neighboring consonants in the chart equal to 1 in most cases, we proceeded similarly also in case of vowels. Thus, the height of the trapezoid was set to 2 as well as its width in the middle. Contrarily to consonants, vowels change more continuously and their coordinates need not be integer values. In total, we represented 64 vowels using the following features:
\begin{enumerate}
	\item{{\em Articulation zone}. For vowels, there are three basic articulation zones: front, central and back. The values are from the interval $\left[-\frac{5}{3},1\right]$. In the middle of the trapezoid's height, the three main articulation zones are represented by values -1, 0 and 1 (see Figure \ref{FigVowels}).}
	\item{{\em Openness}. The range is $\left[-1,1\right]$ and the basic values are open (-1), mid (0) and close (1). }
	\item{{\em Roundness}. This parameter says whether the mouth is rounded or not, so its value is either 0 or 1.}
	\item{{\em Length}. We used only two values for this parameter: 0 (short) and 1 (long).}
	\item{{\em Nasal feature}. This is again 0 or 1, determined by whether the vowel is nasal or not.}
\end{enumerate}

\begin{figure}[h]
	\centering
	\includegraphics[width=0.35\textwidth]{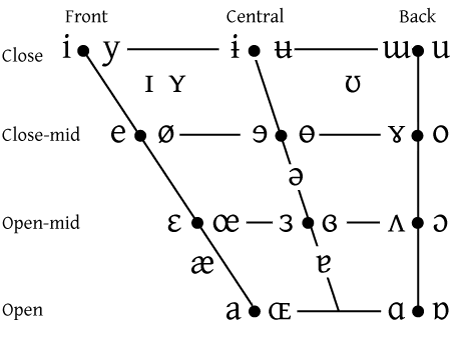} \hspace{1cm}
	\includegraphics[width=0.35\textwidth]{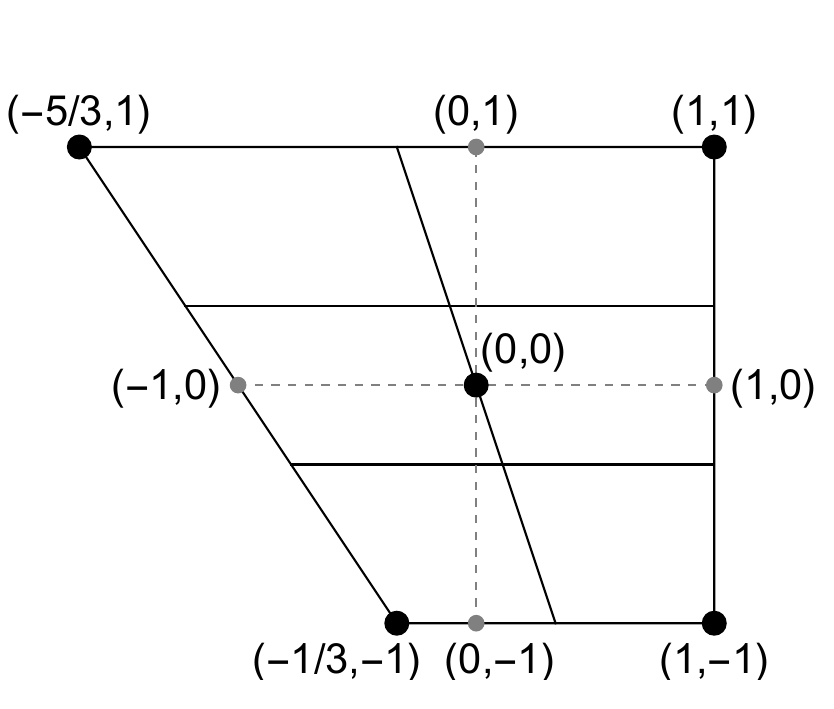} 
	\caption{The vowel chart created by International Phonetic Association \cite{IPA}. Most vowels are listed in pairs, where the one on the right represents the round version of the sound. In our model, the first two coordinates of a vowel are set so that the vowel `\textipa{@}' is placed at the origin and the coordinate axes intersect the trapezoid at $(-1,0)$, $(1,0)$, $(0,-1)$ and $(0,1)$.} 
	\label{FigVowels}
\end{figure}

Now, let $f_{con}^i$ be the $i$-th consonant feature, $i=1,\dots 6$, and let $x$, $y$ be two consonants. We define the consonant distance $\rho_{con}$ and the substitution weight $\sigma_{con}$ as
\begin{equation}
	\rho_{con}(x,y)=\sum\limits_{i=1}^6 |f_{con}^i(x)-f_{con}^i(y)|, \quad \sigma_{con}(x,y)=a_{con}\rho_{con}(x,y), 
	\label{EqConsonantDist}
\end{equation}
where $a_{con}$ is a constant from the interval $[0,1]$. 
Similarly, if $f_{vow}^i$ is the $i$-th vowel feature, $i=1,\dots,5$, and $x$, $y$ are two vowels, the vowel distance $\rho_{vow}$ and the corresponding substitution weight $\sigma_{vow}$ are defined as
\begin{equation}
	\rho_{vow}(x,y)=\sum\limits_{i=1}^5 |f_{vow}^i(x)-f_{vow}^i(y)|, \quad \sigma_{vow}(x,y)=a_{vow}\rho_{vow}(x,y). 
	\label{EqVowelDist}
\end{equation}
 
\section{Experiments}

In this section, we present results of various experiments that we performed in order to set the model parameters and also to analyze some to real-world data. The language data that we used are available for viewing, downloading and using at Zenodo (\url{https://doi.org/10.5281/zenodo.15168042}).
 
\subsection{Setting the model parameters}
 
Summarizing the equations (\ref{EqFunctionf}), (\ref{Eqd0}), (\ref{EqConsonantDist}) and (\ref{EqVowelDist}), we can see that our model uses several parameters. We will now explain how exactly we set them.

\subsubsection{Scaling parameters for the substitution weight} 

First, let us explain the choice of the substitution weight scaling parameters $a_{con}$ and $a_{vow}$. Since the other elementary edits have all weights equal to 1, we want the substitution weights to be less than 1 for very similar phonemes and significantly greater than 1 for very distant phonemes. Also, since vowels are changed more easily than consonants, we want the maximum vowel substitution weight to be lower than the maximum consonant substitution weight. The minimum consonant distance in our consonant representation is 1, while the maximum is 20. For vowels, the minimum distance is $0.25$ and the maximum is $\frac{23}{3}$. If we used $a_{con}=1$ and $a_{vow}=1$, the minimum and maximum substitution weights would be equal to the minimum and maximum distances stated above. However, such values do not quite meet our criteria, especially for consonants, since they are too high and no consonant substitution would have a weight less than 1. On the other hand, we can see that if we chose $a_{vow}\leq\frac{3}{23}$, no vowel substitution would have a weight greater than 1. Based on these facts and after various tests, we chose $a_{con}=a_{vow}=0.3$. This yields consonant substitution weights from the interval $\left[0.3,6\right]$ and vowel substitution weights from $\left[0.075,2.3\right]$.

\subsubsection{Consonant and vowel weights}

The next parameters to determine are the consonant weight $w_{con}$ and the vowel weight $w_{vow}$ in the definition of $d_0$ (\ref{Eqd0}). Intuitively, one would expect $w_{vow}$ to be lower than $w_{con}$, however, the optimal values are not obvious. In order to find them, we decided to use an optimization method.

Let us consider several reference language clusters and a classified language that belongs to one of them. In order to find an optimal way of measuring the distance between words, we need to find the parameter values that identify the correct cluster with the highest probability possible. With respect to this purpose, taking in account the similarities between the reference languages might create a situation that is a little too complex. Thus, in this step of tuning our algorithm, we used a simplified classification procedure that only needs to know the cluster where each reference language belongs. This was the equation (\ref{EqLaplaceFinal}), but with an additional boundary condition $\varphi({v}_i)=\iota_{c_i}$, $i=1,\dots,n$, where $c_i$ is the index of the reference cluster that contains the language represented by $v_i$. We used one classified language and three reference clusters: Slavic, Romance and Germanic. Each of them contained 5 languages which were Czech, Polish, Russian, Serbian, Slovak, Catalan, French, Italian, Portuguese, Spanish, Dutch, German, English, Norwegian and Swedish.

We performed the optimization for a total of 16 classified languages, one language at a time. We used only languages that belonged to one of the three clusters, namely: Belarusian, Bulgarian, Lower Sorbian, Macedonian, Slovenian, Ukrainian, Upper Sorbian, Latin, Occitan, Romanian, Romansh, Afrikaans, Danish, Faroese, Frisian (West) and Icelandic. For each language, we used 12 concepts and only one term per concept. We chose the concepts from the Swadesh list with two conditions: maximum two cognate groups per concept in each reference language cluster (to avoid too complex situations) and a notable presence of sound changes across languages (for the algorithm to be challenged). The concepts that we selected were: to bite, blood, to blow, breast, earth, five, head, to live, mouth, sky, tongue, wing.

For a given language and the $i$-th concept, let $\Phi_i(w_1,w_2)$ be the solution of our simplified diffusion equation with $w_{con}=w_1$ and $w_{vow}=w_2$. For each classified language, we solved the optimization problem
\[ \min\limits_{(w_1,w_2)\in[0,1]^{2}}\left\lVert\frac{1}{12}\sum\limits_{i=1}^{12}\Phi_i\left(w_1,w_2\right)-\iota_{k}\right\rVert. \]
The method that we used for optimization was SOMA \cite{Zel} with the population of 30 and 5 iterations. We performed the optimization for 6 different values of $K$: 0.5, 0.6, 0.7, 0.8, 0.9, 1.0. The final values of parameters were obtained by averaging the results from all 96 runs of SOMA. They were (rounded to one decimal place): $w_{con}=1.0$, $w_{vow}=0.7$. These are the values that we used in all experiments presented further in this section.

\subsubsection{The edge weight parameter $\nu_0$}

The definition of the function $\nu$ in Section \ref{SecMathematicalModel} includes a positive parameter $\nu_0$ that is used to set the weights of the edges connecting the reference languages with their corresponding hypothetical languages. This parameter regulates the influence of the boundary condition on the resulting positions of the reference languages. We set its value according to the simple interpretation following from the equalities (\ref{EqSolutionNu0}): it determines how far apart will two clusters with identical translations be from each other in the resulting Euclidean graph. For $\nu_0\rightarrow 0$, the position of each cluster will converge to the position of its hypothetical language. For $\nu_0\rightarrow\infty$, the clusters will approach the arithmetic mean of $\iota_1$ and $\iota_2$. 

An appropriate value of $\nu_0$ is a value that places clusters with identical translations close enough to each other but it does not annihilate the influence of the boundary condition. Our condition for the distance of the two clusters was
\[ \Vert \varphi_1-\varphi_2\Vert \leq \frac{1}{10}\Vert \iota_1-\iota_2\Vert. \]
Combining this condition with (\ref{EqSolutionNu0}), we get $\nu_0\geq \frac{9}{2}$. 
Based on this inequality, we set $\nu_0=5$ for all of our experiments. As we will see in the presented results, this value matches our criteria -- it brings together very similar translations but keeps apart those that are not similar. 

\subsubsection{The diffusion parameter $K$}

Finally, the question is how to set the parameter $K$ that determines the diffusion coefficients. In the first round of the tests, we wanted to find out which values of $K$ sufficiently correctly classify a given word. We used the model (\ref{EqLaplaceFinal}) -- (\ref{EqGraphLaplaceFinal}) and tried to classify 25 words that were from 11 Indo-European languages. The reference languages were the same as in the previous experiment and we chose words that clearly belong to only one of the clusters.  However, in order to properly examine which values of $K$ work, the words were of various level of difficulty -- variously distant from their cognates in the cluster. In the easiest group, there was, for example, the Afrikaans word `bloed' \textipa{[blut]} meaning `blood', which is very similar or identical with all its Germanic equivalents (Dutch `bloed' \textipa{[blut]}, English `blood' \textipa{[bl2d]}, German `Blut' \textipa{[blut]}, Norwegian `blod' \textipa{[blu:]}, Swedish `blod' \textipa{[blu:d]}). Other terms for `blood' were also in the medium difficulty group, for example the Icelandic `bl\' o\textipa{D}' \textipa{[plou:D]}   or Romanian `s\^ange' \textipa{[s1ndZe]}(compare to Catalan `sang' \textipa{[saN]}, French `sang' \textipa{[s\~Ak]}, Italian `sangue' \textipa{[sangwe]}, Portuguese `sangue' \textipa{[s\~5gwi]}, Spanish `sangre' \textipa{[saNgRe]}). For the medium and the high difficulty group, we used also some languages that do not belong to any of the reference clusters, but their words do have a cognate in one of the clusters. Such words are often quite different from their cognates, though still recognizable, so we were curious how our algorithm would deal with that. An example of such a word is Albanian `harabel' \textipa{[haRabEl]}, meaning `sparrow', cognate with Czech `vrabec' \textipa{[vrabEts]}, Polish `wr\' obel' \textipa{[vrubEl]}, Russian `\foreignlanguage{russian}{воробей}' \textipa{[v9r5bej]}, Serbian `\foreignlanguage{russian}{врабац}' \textipa{[Vrabats]} and Slovak `vrabec' \textipa{[vrabets]}. Another example is Lithuanian `žvaigždė' \textipa{[ZV5jgZde:]}, meaning `star', cognate with Czech `hv\v ezda' \textipa{[Hvjezda]}, Polish `gwiazda' \textipa{[gvjazda]}, Russian `\foreignlanguage{russian}{звезда}' \textipa{[zvIzda]}, Serbian `\foreignlanguage{russian}{звезда}' \textipa{[zVe:zda]} and Slovak `hviezda' \textipa{[hvjezda]}. 

Our criterion for a successful classification was as least 80\% similarity with the correct reference cluster. The lowest value of $K$ for which this was achieved varied between 0.25 and 0.75, with the median equal to 0.45. With increasing $K$, the probability was increasing as well. 

At this point, we have identified the parameter values that reliably detect similarities, which is important for our overall procedure of lexical similarity evaluation. However, in our application, we are not necessarily interested in the highest probability classification. In real world situations, the reference languages are never strictly separated and many times, we find words similar to the classified word in more than just one cluster. In such cases, we rather want a classification that also reflects the relations between the reference languages. As we will show, the parameter $K$ is the one that is used to tune the algorithm for exactly this purpose. Let us now explore how it influences the process and which values are the most suitable for capturing the mutual similarities. 

For testing, we chose the concept `flower'. The corresponding terms are listed in Table \ref{TabFlower}. As we can see, there is one outlier among the Germanic words -- the English word `flower'. This comes from Latin `fl\=os' (meaning `flower') and it is a close cognate with all Romance terms shown in the table. Moreover, to make the situation a little more intricate, all Romance and Germanic terms are considered to be derived from one Proto-Indo-European root `\textipa{*b\super{h}leh\textsubscript{3}}' and there is actually some non-negligible phonetic similarity among all of them. The Slavic terms have a completely different origin.

\begin{table}
	\centering
	\begin{tabular}{l*{3}{l}}
		\\
		Slavic & Romance & Gemanic  \\
		\hline \vspace{-0.1cm}\\
		kv\v et \textipa{[kvjEt]} (Czech)	 & flor \textipa{[flO]} (Catalan)  & bloem \textipa{[blum]} (Dutch)   \\
		kwiat \textipa{[kvjat]} (Polish)	 & fleur \textipa{[fl\oe K]} (French) & flower \textipa{[flaU\*r]} (English)  \\
		\foreignlanguage{russian}{цветок} \textipa{[cvItok]} (Russian) & fiore \textipa{[fjore]} (Italian) & Blume \textipa{[blume]} (German)   \\
		\foreignlanguage{russian}{цвет} \textipa{[cvEt]} (Serbian) & flor \textipa{[floR]} (Portuguese) & blomst \textipa{[blOmst]} (Norwegian) \\
		kvet \textipa{[kvet]} (Slovak) & flor \textipa{[floR]} (Spanish) & blomma \textipa{[blUma]} (Swedish) \
	\end{tabular}
	\caption{The terms for the concept `flower' in 15 reference languages.\\}
	\label{TabFlower}
\end{table}

For classification, we used the Scots term `flour' \textipa{[flu:r]}, the only other Germanic term resembling `flower'. In Table \ref{TabFlowerClass}, we can see its similarity distribution for various values of $K$. As we can see, for $K=3.0$, the algorithm is already very categorical and basically says only which cluster is the most similar to the Scots word. For smaller values of $K$, we get more realistic and structured results. This can be observed also in Figure \ref{FigFlowerClass}, where we show results for three different values of $K$. As the definition of the function $g$ (\ref{EqFunctionf}) suggests, the value of $K$ determines how distant two words can be to still have a non-negligible impact on each other. Thus, for a low $K$ ($K=0.4$), we can see that all pairs of words are visibly attracted to each other (there is at least a little deviation from their hypothetical language). For $K=0.6$, we can still observe some attraction between the Germanic and the Romance words, which reflects their common Proto-Indo-European root. The Romance words are, in addition, quite strongly mutually attracted with the English `flower', which reflects their common Latin origin. The Slavic languages that are unrelated stay very near their hypothetical language.  For $K=1.0$, we can observe that the attraction between the Germanic and the Romance words is practically not present anymore, except for the word `flower', which is weakly attracted. Based on all observation discussed in this part and many similar tests that we performed, we recommend to choose the value of $K$ from the interval $[0.5,0.8]$. 

\begin{table}
	\centering
	\begin{tabular}{l*{8}{c}}
		\\
		& K=0.4 & K=0.5 & K=0.6 & K=0.7 & K=0.8 & K=1.0 & K=$3.0$ \\
		\hline \vspace{-0.1cm}\\
		Slavic & 0.040 & 0.013 & 0.004 & 0.001 & 0.0 & 0.0  & 0.0 \\
		Romance & 0.729 & 0.802 & 0.847 & 0.881 & 0.909 & 0.949 & 1.0 \\
		Germanic & 0.231 & 0.185 & 0.149 & 0.118 & 0.091 & 0.051 & 0.0 \\
	\end{tabular}
	\caption{The similarity distribution for the Scots term `flour' (flower) for various values of $K$.}
	\label{TabFlowerClass}
\end{table}

\begin{figure}[h]
	\centering
	\includegraphics[width=0.3\textwidth]{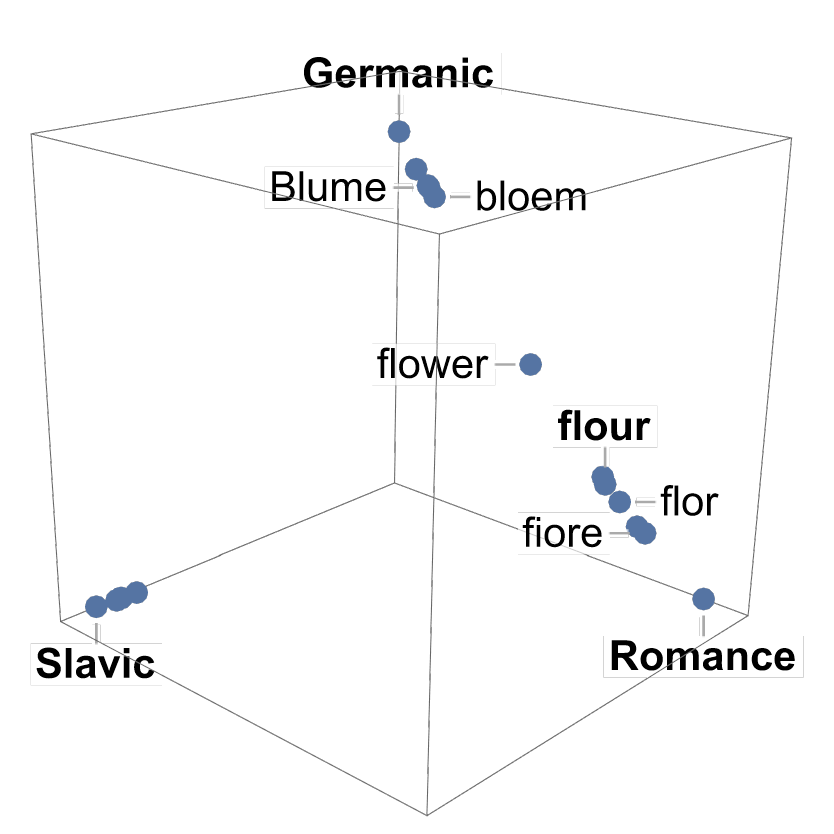} 
	\includegraphics[width=0.3\textwidth]{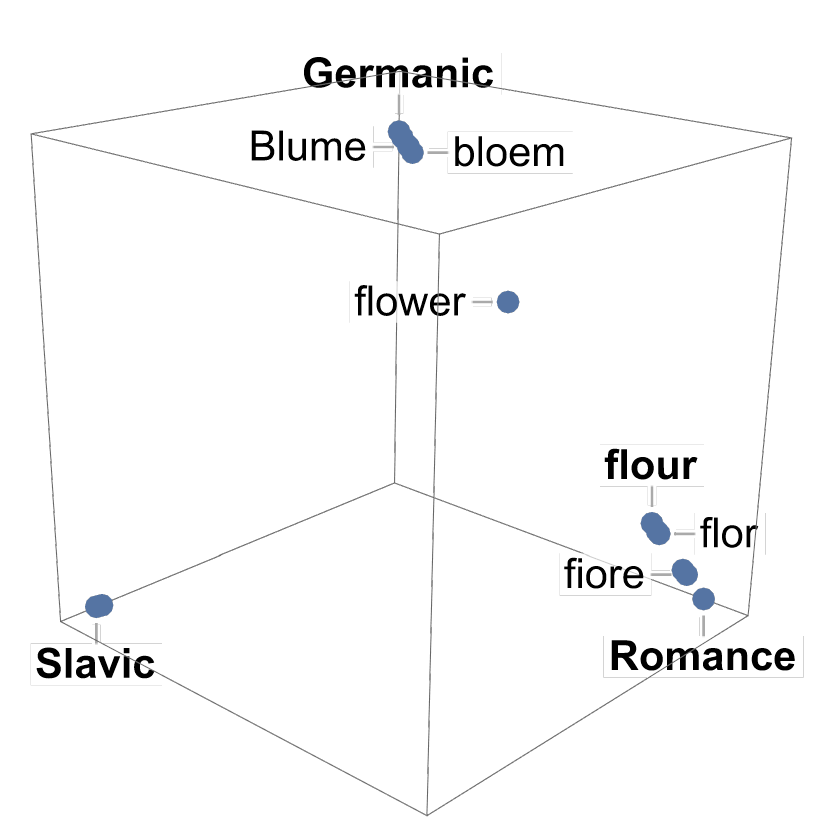} 
	\includegraphics[width=0.3\textwidth]{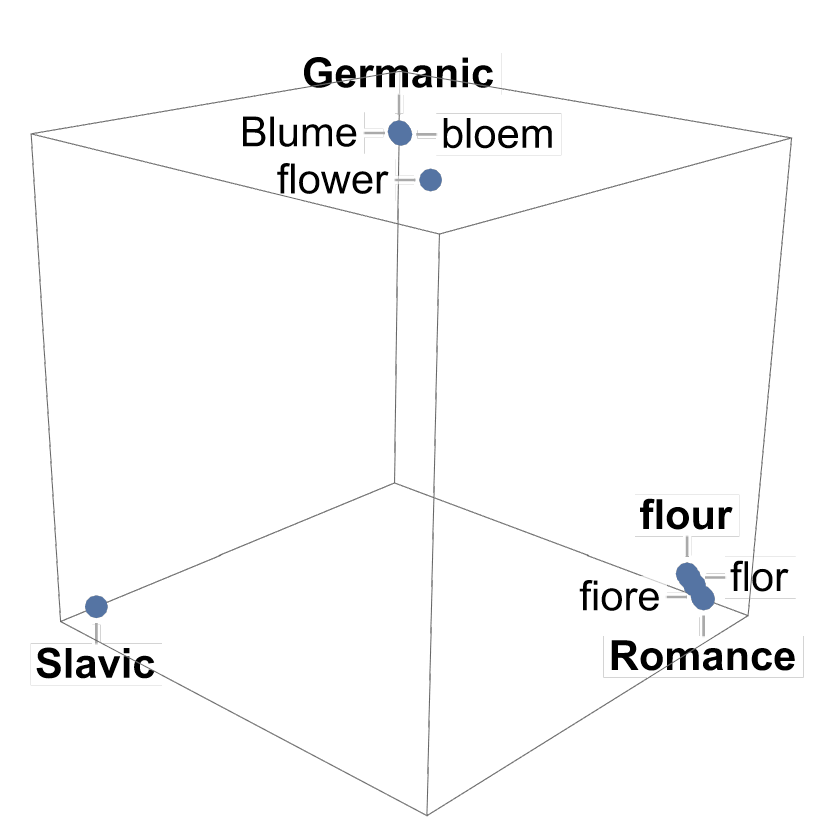} 
	\caption{The classification of the Scots word `flour' for different values of the diffusion parameter $K$. The points `Slavic', `Romance' and `Germanic' represent the hypothetical languages. On the left, we used the value $K=0.4$. We can observe that all words are, to some extent, mutually attracted (they diverted at least a little from their hypothetical language). For $K=0.6$ (middle), we can see that Germanic words are still slightly attracted to their Romance counterparts, which reflects their distant common origin. The only exception is the word `flower' which exhibits a quite strong mutual attraction with its Romance cognates. In the third picture, we set $K=1.0$. This value is already a little too high, since the result only weakly reflects the relationships between the words in the experiment. }
	\label{FigFlowerClass}
\end{figure}

\subsection{A case study}\label{SecCaseStudy}

Having set our model, we can move on to inspecting some real language vocabularies. For our tests, we decided to explore the vocabulary related to traditional tools. This is a part of language where foreign influences are very common and some languages use borrowings from several different language groups.

In order to perform the experiments, we created a list of 42 words containing traditional hand tools, agricultural and outdoor tools, binding tools and the basic shed equipment. The list is presented in Table \ref{TabTools}.

\begin{table}
	\vspace{0.4cm}
	\centering
	\begin{tabular}{*{7}{l}}  
	\multicolumn{2}{c}{Hand tools} & \multicolumn{2}{c}{Agriculture \& outdoors} & Binding & \multicolumn{2}{c}{Shed equipment} \vspace{0.1cm}  \\ 
	\toprule \vspace{-0.2cm}\\
	anvil & plane & adze & rake & belt/strap & bucket & hook \\
	chisel & pliers & axe & scythe & chain & basket & latch \\
	file & saw & flail & shovel & rope & chest & shed \\
	gimlet & staple & hoe & sickle & & crate & shelf \\
	hammer & vise & pickaxe & spade & & cupboard & table \\
	knife & wedge & pitchfork & stick & & & \\
	nail & whetstone & plow & yoke & & & \\
	& & plowshare & & & &  \\
	\end{tabular}
	\caption{The list of 42 concepts used in our case study. \\}
	\label{TabTools}
\end{table}

The languages that we chose to explore were eight European languages from different parts of the continent: Albanian, Aromanian, Breton, Estonian, Latvian, Lithuanian, Maltese and Romanian. All these languages are spoken in linguistically diverse territories and were also subject to various influences in the history. For lexical similarity evaluation, we assembled several reference clusters:
\begin{enumerate}
	\item{Balkans-related (Greek, Hungarian, Turkish),}
	\item{Baltic (Latvian, Lithuanian),}
	\item{Celtic (Cornish, Welsh),}
	\item{Germanic (Danish, Dutch, English, German, Norwegian, Swedish),}
	\item{Romance (Catalan, French, Italian, Latin, Spanish, Portuguese),}
	\item{Slavic I (Bulgarian, Macedonian, Russian, Serbian, Slovak, Ukrainian),}
	\item{Slavic II (Belarusian, Polish, Russian, Serbian, Slovak, Ukrainian).}
\end{enumerate}
In some experiments, we also used one-language clusters, which have the same name as the language that they contain.

In all experiments included in our case study, we used $K=0.6$ and we allowed synonyms. The translations of the concepts were obtained using numerous online dictionaries and cross-checked, if possible. Since many of the tools have dedicated Wikipedia pages in various languages, we also used those. The phonetic transcriptions were obtained by means of the  OpenL translator \cite{OpenL}. However, for some languages, significant manual corrections were necessary. These were based on the phonology and pronunciation rules of each language. As for the system of linear equations (\ref{EqLinSystem}), we solved it by LU-decomposition as implemented in the Eigen library \cite{Eigen}.

Contrarily to the test examples presented in the previous parts, the real world relationships between languages bring along reference clusters of various sizes. As we explained in Section \ref{SecMathematicalModel}, our model (\ref{EqLaplaceFinal}) -- (\ref{EqGraphLaplaceFinal}) should be able to handle such situations and not to unfairly incline to reference clusters of a bigger size. To test our approach, we performed various experiments with words that have cognates in several differently sized clusters. Here, we present two examples. 

In the first example, we classified the Estonian word `viil' \textipa{[vi:l]}, meaning `file' (the tool). The reference clusters were Slavic II, Germanic and Finnish. The Estonian word is cognate and phonetically close to the Finnish term `viila' \textipa{[Vi:la]}, but also, more or less, to all six Germanic terms: Danish and Swedish `fil' \textipa{[fi:l]}, Dutch `vijl' \textipa{[vEil]}, English `file' \textipa{[faIl]}, German `Feile' \textipa{[faIl@]} and Norwegian `fil' \textipa{[fil]}. The Slavic terms are etymologically related but phonetically much more distant. In the first test, we tried the basic model (\ref{EqLaplace}) -- (\ref{EqGraphLaplace}). The resulting feature coordinates for Estonian were $(0.03,0.82,0.15)$, which shows that the similarity with Finnish was unrealistically shadowed by the Germanic cluster. Using the model (\ref{EqLaplaceFinal}) -- (\ref{EqGraphLaplaceFinal}), we obtain the feature coordinates $(0.01,0.46,0.53)$, which  reflects the situation much more truthfully. In Figure \ref{FigClustersSize1}, we can see the resulting placement of all words from both experiments in the feature space.

\begin{figure}[h]
	\centering
	\includegraphics[width=0.4\textwidth]{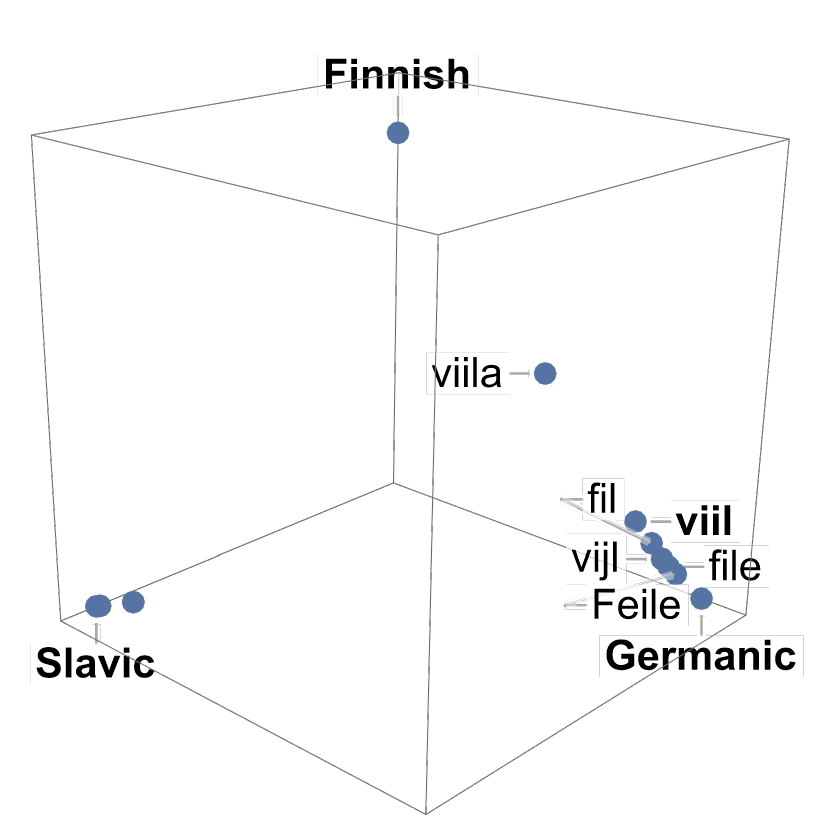} 
	\includegraphics[width=0.4\textwidth]{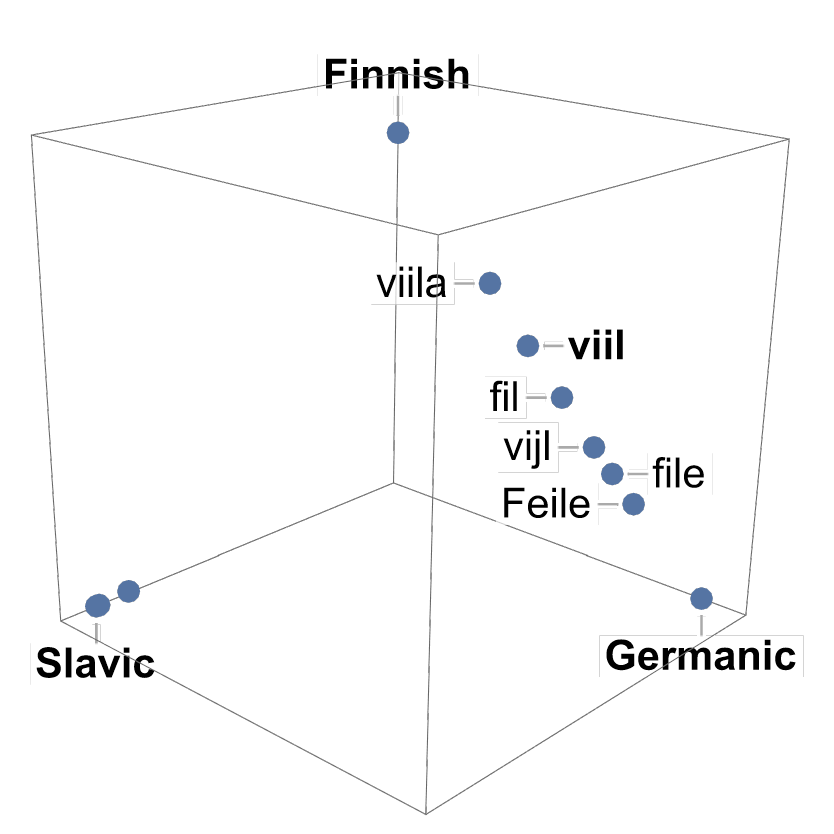} 
	\caption{Classification of words that have phonetically similar counterparts in several reference clusters of varying size: the Estonian `viil' (file). The basic model (\ref{EqLaplace}) -- (\ref{EqGraphLaplace}) causes an unrealistic dominance of the Germanic cluster (left). The directed graph model (\ref{EqLaplaceFinal}) -- (\ref{EqGraphLaplaceFinal}) fairly evaluated the match with each cluster and did not prioritize clusters of a bigger size (right).}
	\label{FigClustersSize1}
\end{figure}

In the second example, we classified the Romanian word `coas\v a' \textipa{[koas@]}, meaning `scythe'. The reference clusters were Slavic I, Balkans-related and Albanian. This time, the Romanian word had a very similar counterpart in all clusters: Bulgarian,  Macedonian, Russian, Serbian and Ukrainian `\foreignlanguage{russian}{коса}', pronounced \textipa{[kOs5]}, \textipa{[kosa]}, \textipa{[k5sa]}, \textipa{[kOsa]} and \textipa{[kOsA]}, Slovak `kosa' \textipa[{kosa}], Greek $\kappa\acute{o}\sigma\mkern-2mu\alpha$ \textipa{[kosa]}, Hungarian `kasza' \textipa{[k6s6]} and Albanian `kos\"e' \textipa{[kos@]}. The only exception was the Turkish `t\i rpan' \textipa{[tWRpan]}. The feature coordinates computed by the basic model (\ref{EqLaplace}) -- (\ref{EqGraphLaplace}) were $(0.67, 0.22, 0.11)$. Again, the largest cluster (Slavic) is unrealistically dominant. What we would expect in this case is similarity distribution with approximately equal values for the first and the third cluster and a lower value for the second cluster, equal to approximately 2/3 of the other two values. The feature coordinates computed by the directed graph model (\ref{EqLaplaceFinal}) -- (\ref{EqGraphLaplaceFinal}) were $(0.37,0.25,0.38)$ and they exactly meet this expectation. The results of both tests are visualized in Figure  \ref{FigClustersSize2}.

\begin{figure}[h]
	\centering
	\includegraphics[width=0.4\textwidth]{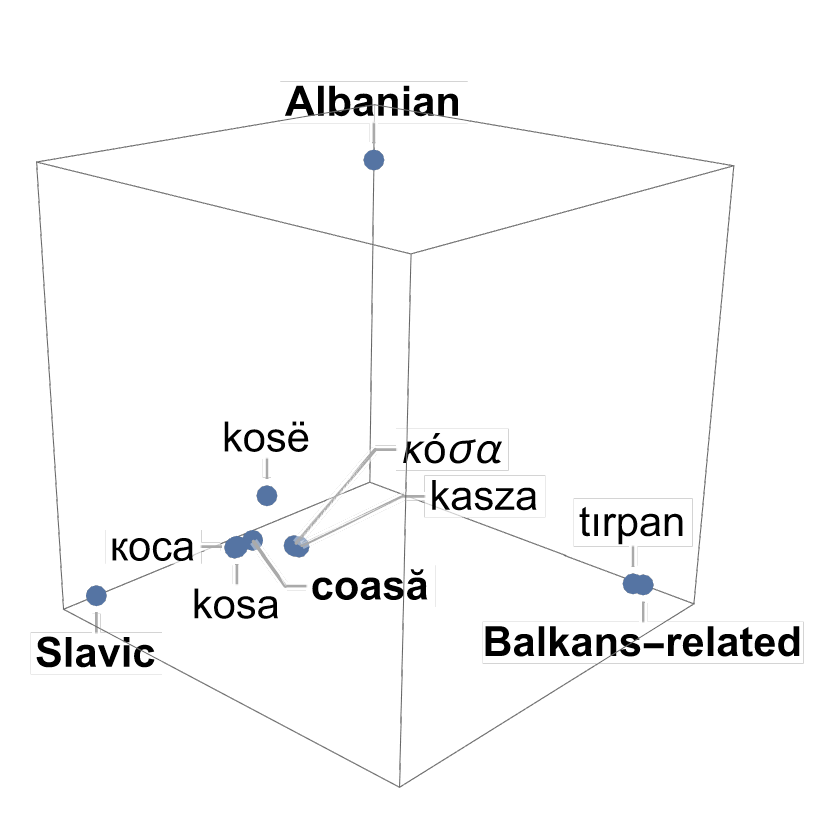} 
	\includegraphics[width=0.4\textwidth]{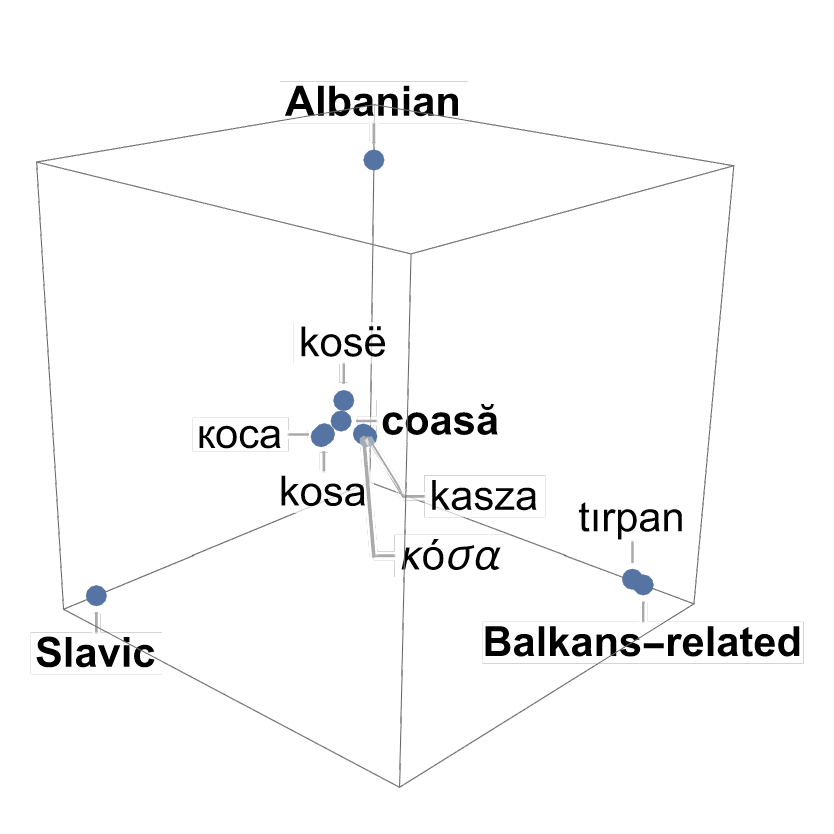} 
	\caption{Classification of words that have phonetically similar counterparts in several reference clusters of varying size: the Romanian `coas\v a' (scythe). Again, the directed graph model (\ref{EqLaplaceFinal}) -- (\ref{EqGraphLaplaceFinal}) provided much more realistic results (right) than the basic undirected graph model (left).}
	\label{FigClustersSize2}
\end{figure}

Before moving on to the main experiments, let us explain how they should be interpreted. For this purpose, we classified two languages: Czech and Arabic, both with reference clusters Germanic, Romance and Slavic I. By averaging the similarity distributions of the individual words, we obtained the overall lexical similarity distributions: $(0.86, 0.04, 0.1)$ for Czech and $(0.3, 0.36, 0.34)$ for Arabic. A plot with the average probabilities is shown in Figure \ref{FigCzechArabic}. For Czech, the interpretation of the probabilities is quite straightforward, since a high similarity with the Slavic cluster is observed. However, in case of Arabic, one could ask what exactly the numbers tell -- they could mean approximately one third of common words with each cluster but also no similar words at all. Therefore, with each experiment, we also display a histogram of normalized minimal word distances (NMWD). For a given classified word, the NMWD is the distance from its most similar counterpart among all reference words, divided by the average word length of the classified language. The histograms corresponding to Czech and Arabic are shown in Figure \ref{FigCzechArabic} on the right. As we can see, for Czech, the peak of the histogram is close to zero and the NMWD distribution indicates that almost each word has a very similar counterpart in the reference dataset. Thus, the computed average probabilities can be read as a real lexical similarity distribution. Arabic, on the other hand, has a peak quite far from zero and almost no words that are highly similar to any of the reference words. Thus, the average probabilities should be interpreted rather as an equal dissimilarity with all reference clusters. If there was an equal high similarity, we would get similar average probabilities but a histogram resembling the one that represents Czech.

\begin{figure}[h]
	\centering
   	\parbox{\textwidth}{
   		\parbox{.6\textwidth}{%
      		\centering
      		\includegraphics[width=0.4\textwidth]{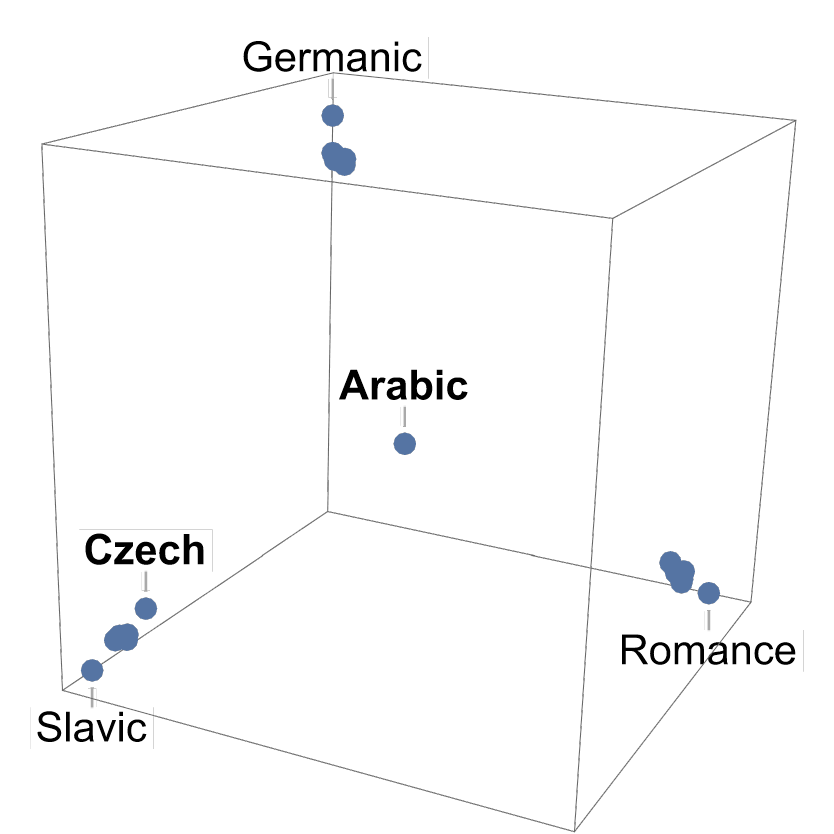}
    	}
   	\parbox{.3\textwidth}{%
      		\includegraphics[width=0.4\textwidth]{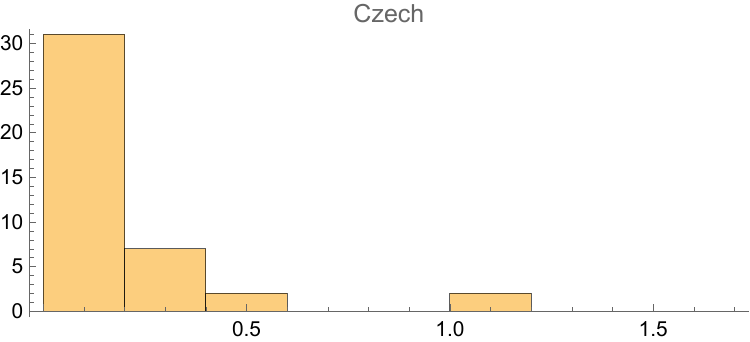}
      		\vskip1em
      		\includegraphics[width=0.4\textwidth]{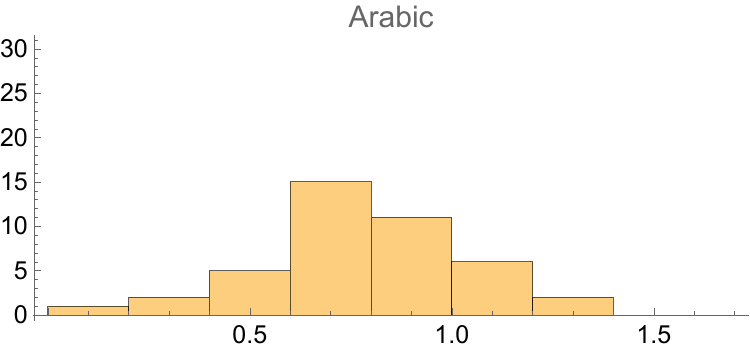}  
    	}
    	}
    	\caption{Classification of Czech and Arabic. On the left, we display the average feature space coordinates. On the right, we show the histograms of the normalized minimal word distance. For a language that finds a lot of similar words in the reference clusters, the histogram peaks close to zero. For a language that is dissimilar to all reference languages, the histogram peaks far from zero. This helps interpret the resulting feature coordinates of the classified language.}
    	\label{FigCzechArabic}
\end{figure}

\subsubsection{The Balkans}

From the linguistic point of view, the Balkans are a very interesting region involving several Slavic and Romance languages and also two unique languages, which are Greek and Albanian. In history, there was a significant Turkish influence due to the Ottoman Empire and also some German and Hungarian influence coming from the Austrian-Hungarian Empire. 

From the interesting mix of languages of the Balkans, we chose three for lexical similarity evaluation: Romanian, Aromanian and Albanian. All three of them are intriguing languages with a lot of unique words and an interesting history that also brought a significant number of borrowings. Since we are talking about a really multilingual territory, we used a quite large number of reference clusters: Slavic I., Romance, Germanic, Greek, Hungarian and Turkish. The results are presented in Figures \ref{FigRomanian} -- \ref{FigBalkans}. As we can see from the NMWD histogram, most Romanian and Aromanian words found a very similar counterpart among the reference languages. Based on this, the algorithm detected a relatively high similarity of Romanian and the Slavic cluster, which was more than double compared to any other reference cluster. For Aromanian, the three most similar clusters were Greek, Slavic and Turkish (in this order), which corresponds to the territory where Aromanians live. For Albanian, the histogram is less dramatic. Nevertheless, its shape indicates that the similarity chart still speaks about real similarity. Compared to Romanian and Aromanian, the similarity distribution is more uniform, with the highest values corresponding to the Slavic, the Romance and the Turkish cluster.

The last figure in this part -- Figure \ref{FigBalkans} -- shows the result of a slightly different experiment. In order to visualize the similarity of the languages in another way, we used only four reference clusters: Slavic I, Romance, Germanic and Balkans-related. This time, we did not compute the average probabilities, but we kept the values for all 42 concepts. Each language thus had $42\cdot 4$ coordinates from the interval $[0,1]$. The picture that we show is the 3D PCA projection of those coordinates.

\begin{figure}[h]
	\centering
	\includegraphics[width=0.4\textwidth]{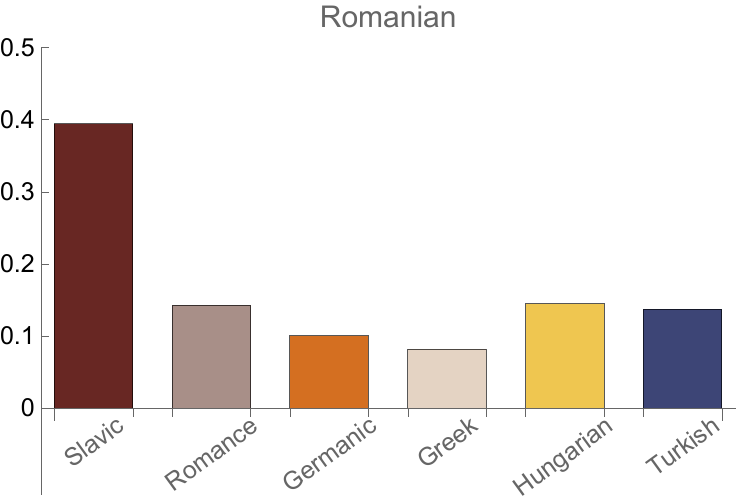} 
	\includegraphics[width=0.4\textwidth]{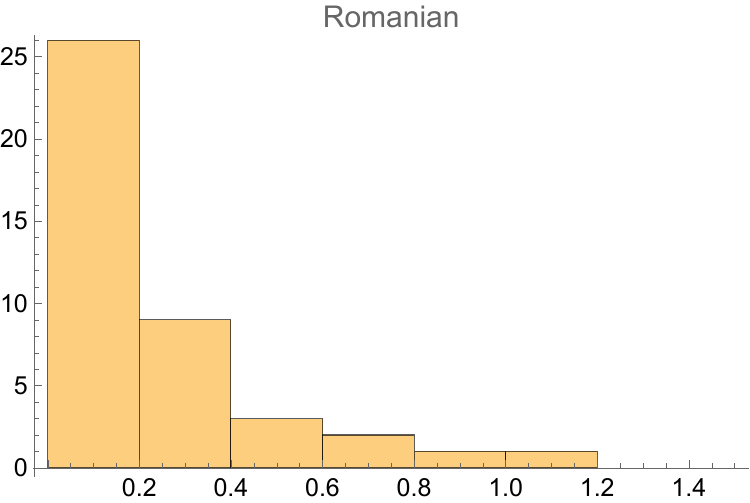} 
	\caption{The similarity distribution and the NMWD histogram for Romanian.}
	\label{FigRomanian}
\end{figure}

\begin{figure}[h]
	\centering
	\includegraphics[width=0.4\textwidth]{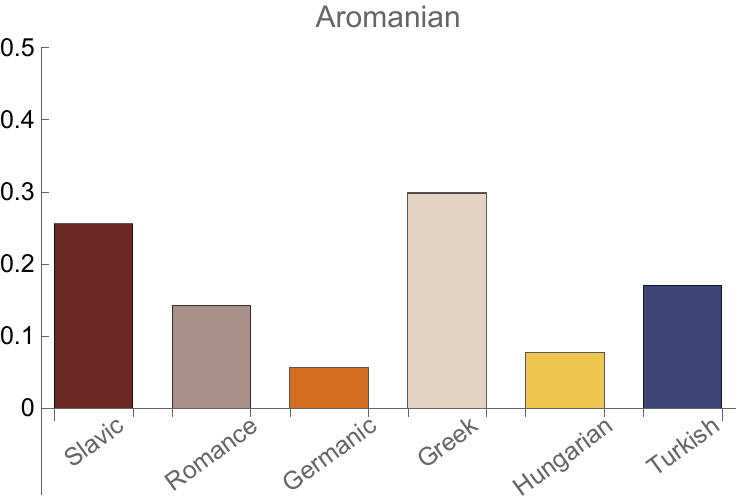} 
	\includegraphics[width=0.4\textwidth]{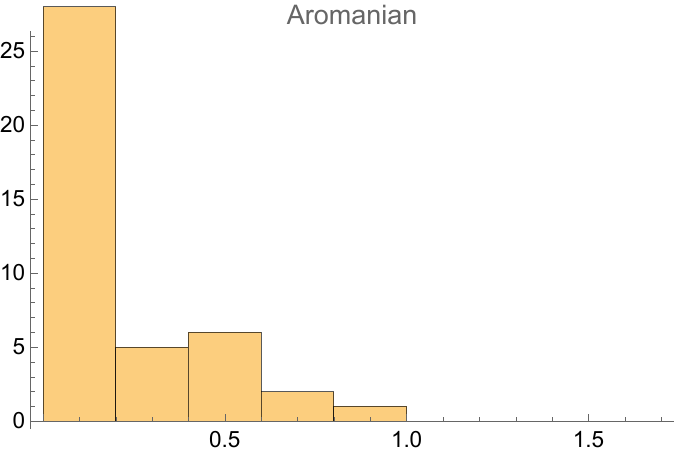} 
	\caption{The similarity distribution and the NMWD histogram for Aromanian.}
	\label{FigAromanian}
\end{figure}

\begin{figure}[h]
	\centering
	\includegraphics[width=0.4\textwidth]{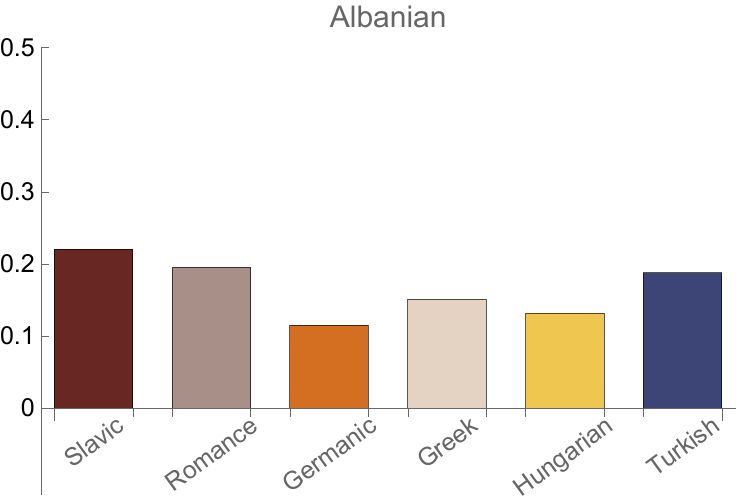} 
	\includegraphics[width=0.4\textwidth]{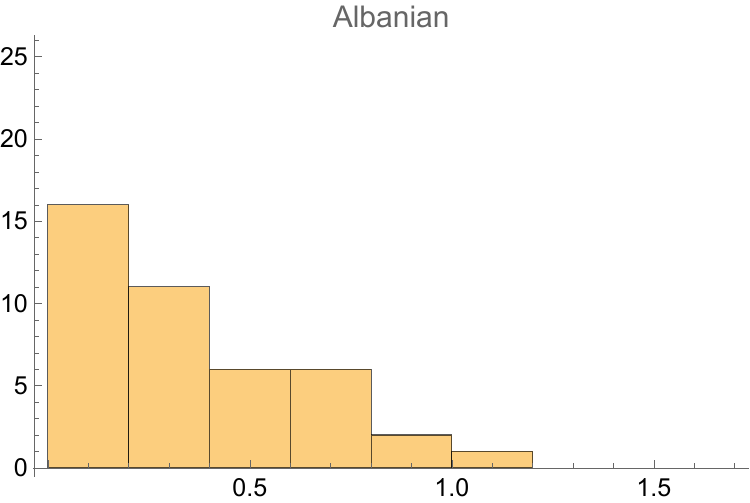} 
	\caption{The similarity distribution and the NMWD histogram for Albanian.}
	\label{FigAlbanian}
\end{figure}

\begin{figure}[h]
	\centering
	\includegraphics[width=0.4\textwidth]{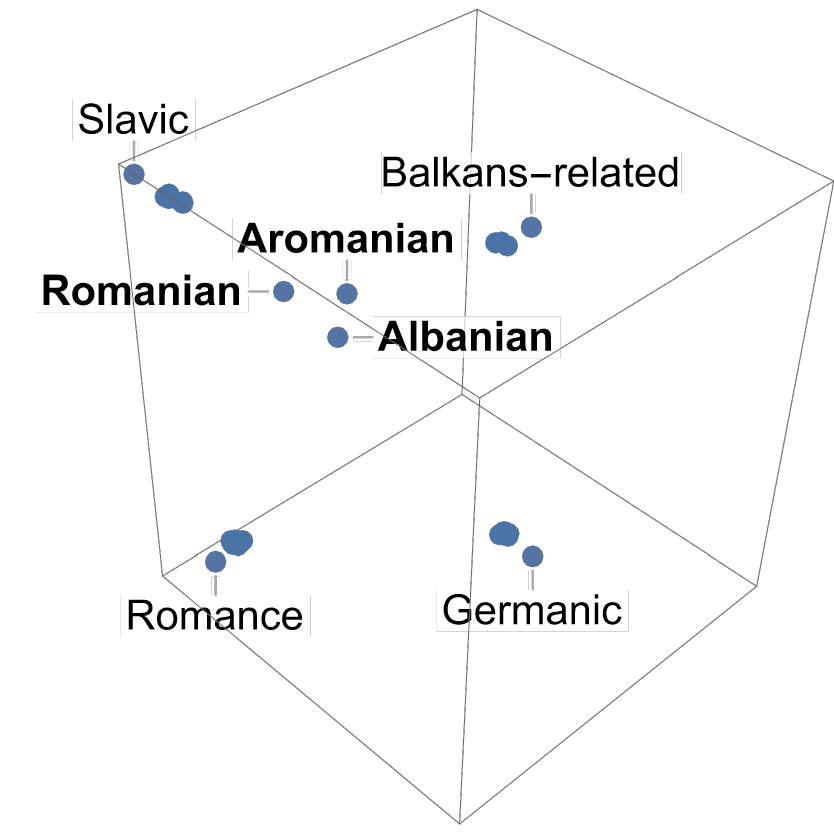} 
	\caption{The classification of Romanian, Aromanian and Albanian. The picture shows the 3D PCA projection of the feature space coordinates.}
	\label{FigBalkans}
\end{figure}

\subsubsection{Baltic languages}

The Baltic region is another interesting territory where several language families meet geographically or historically and one can expect various mutual influences. In order to explore them, we tried to classify all three main languages of the region - Lithuanian, Latvian and Estonian. 

We started with Lithuanian and Latvian and the reference clusters Slavic II, Romance and Germanic, which represent the three main language groups of Europe. The Slavic II cluster was used instead of Slavic I, since Belarusian and Polish are spoken in the direct neighborhood of the region. The results are shown in Figures \ref{FigBaltic} and \ref{FigBalticHistogram}. For Lithuanian, the similarity distribution is not very far from uniform, with some prevalence of the Slavic cluster. However, the corresponding NMWD histogram shows that, in fact, not that many close similarities were detected and that Lithuanian probably has a lot of words that are specific to it. The situation with Latvian is a little different. The most similar cluster seems to be Germanic and the NMWD histogram indicates more matches than the one representing Lithuanian. However, the peak is not right next to zero. Indeed, when we look at the Latvian list of words, we find there quite a lot of words similar to the reference words, but the similarity is more loose than in the previously discussed languages. In the range corresponding to the histogram peak, we find, for example, the words `nazis' \textipa{[nazis]}, meaning `knife', and `lāpsta' \textipa{[la:psta]}, meaning `shovel', that are cognates with Russian \foreignlanguage{russian}{нож}' \textipa{[no\:s]} and `\foreignlanguage{russian}{лопата}'  \textipa{[l5pata]} and similar words across all Slavic languages. In order to classify Latvian more exactly, we ran another experiment with Lithuanian as the fourth cluster. The results are shown in Figure \ref{FigLatvian}. The NMWD histogram shows that this time, most words found a similar counterpart. Based on this, we can conclude that Latvian shares the most of the similar words with the Germanic cluster, and then about an equal amount with the Slavic and Lithuanian cluster.

\begin{figure}[h]
	\centering
   	\parbox{\textwidth}{
   		\parbox{.6\textwidth}{%
      		\centering
      		\includegraphics[width=0.4\textwidth]{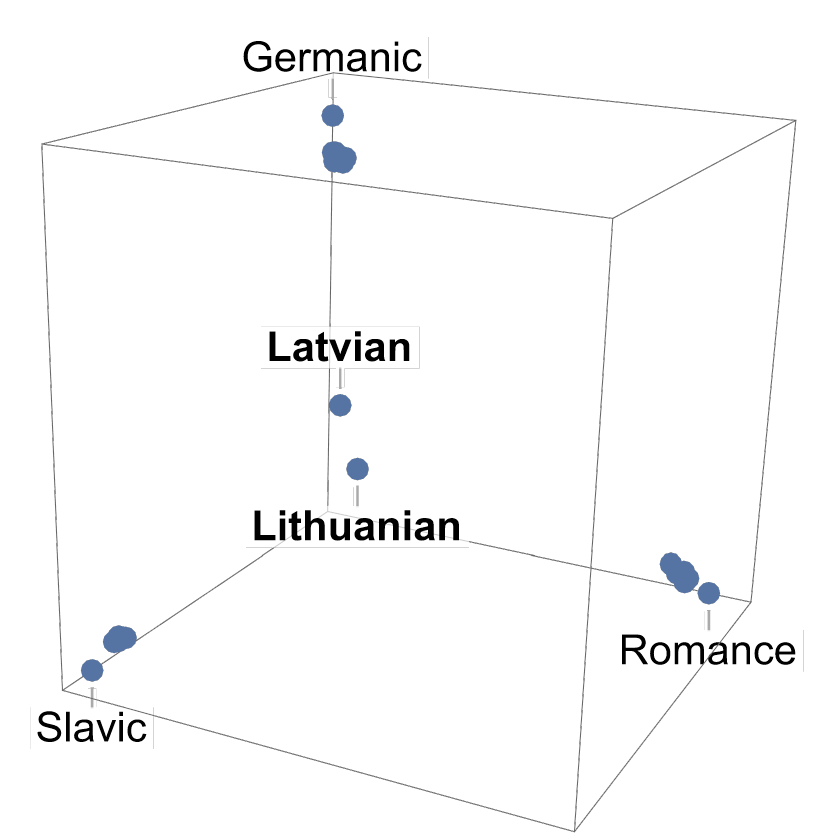}
    	}
   	\parbox{.3\textwidth}{%
      		\includegraphics[width=0.4\textwidth]{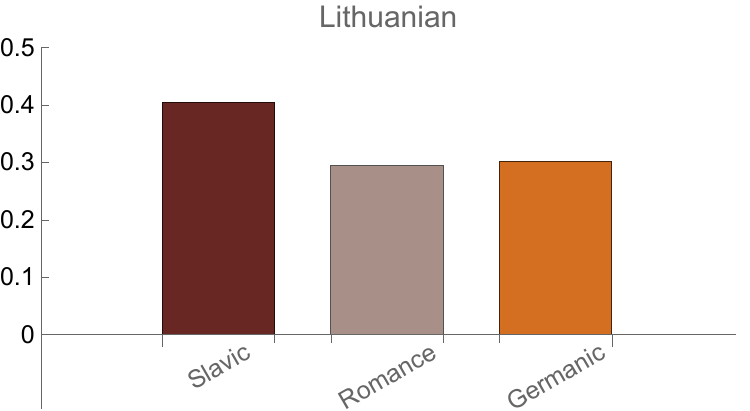}
      		\vskip1em
      		\includegraphics[width=0.4\textwidth]{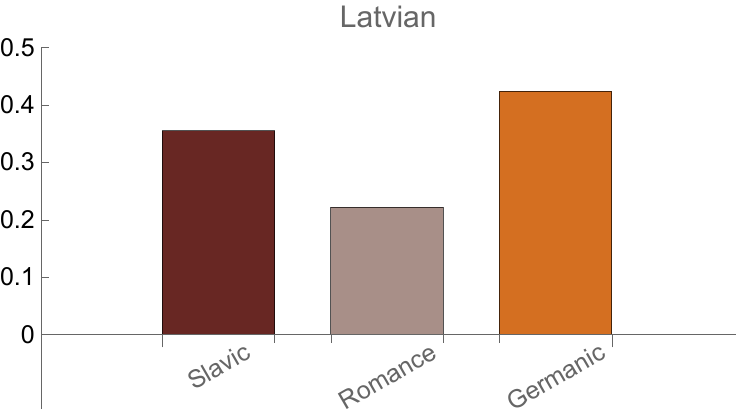}  
    	}
    	}
    	\caption{Classification of Lithuanian and Latvian -- the average feature space coordinates and the corresponding probability distributions.}
    	\label{FigBaltic}
\end{figure}

\begin{figure}[h]
	\centering
	\includegraphics[width=0.4\textwidth]{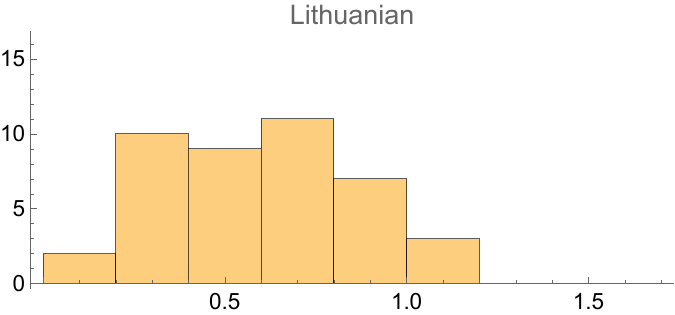} \hspace{0.3cm}
	\includegraphics[width=0.4\textwidth]{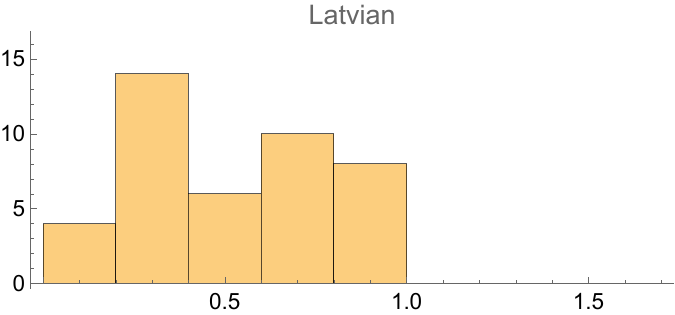} 
	\caption{Classification of Lithuanian and Latvian -- the NMWD histograms.}
	\label{FigBalticHistogram}
\end{figure}

\begin{figure}[h]
	\centering
   	\parbox{\textwidth}{
   		\parbox{.6\textwidth}{%
      		\centering
      		\includegraphics[width=0.4\textwidth]{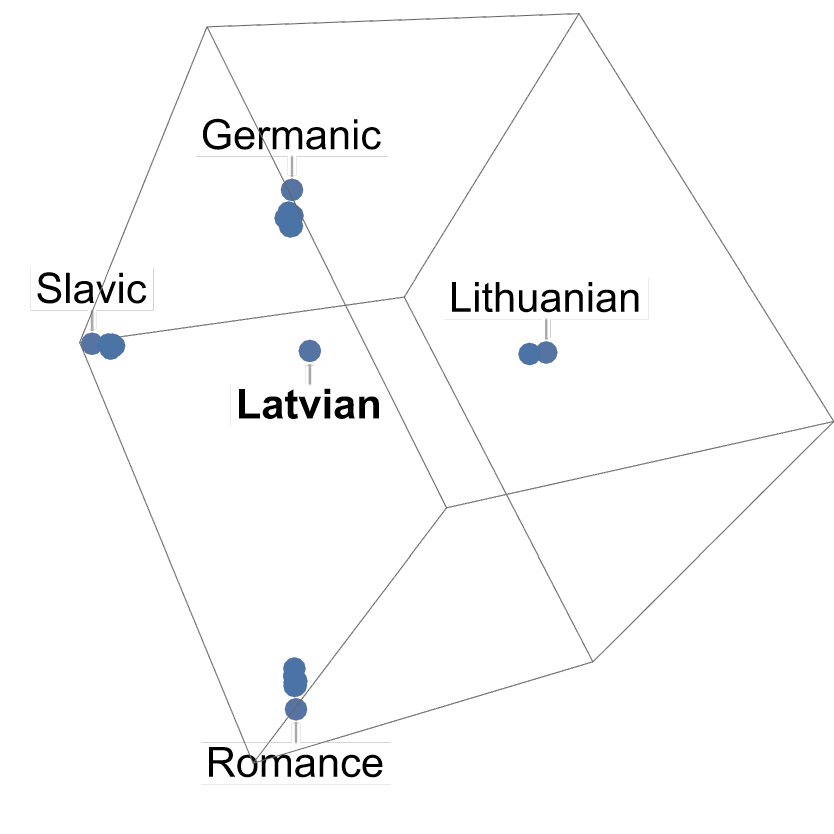}
    	}
   	\parbox{.3\textwidth}{%
      		\includegraphics[width=0.4\textwidth]{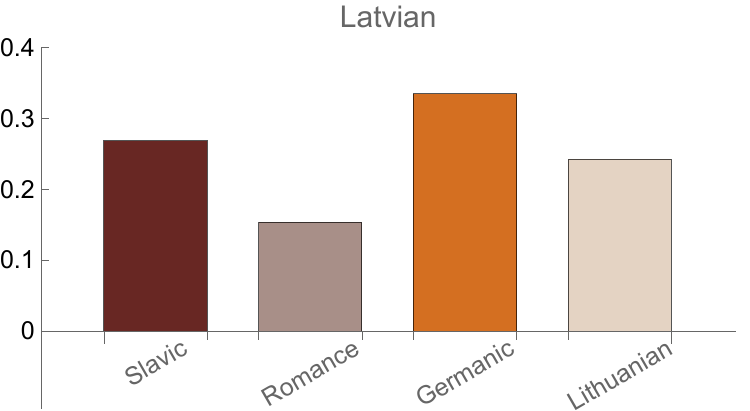}
      		\vskip1em
      		\includegraphics[width=0.4\textwidth]{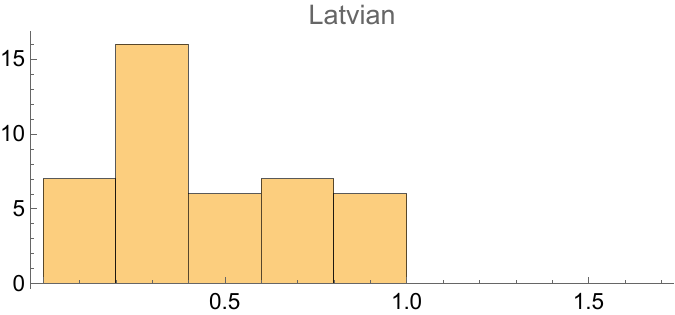}  
    	}
    	}
    	\caption{Classification of Latvian -- the 3D PCA projection of the feature space coordinates, the similarity distribution and the NMWD histogram.}
    	\label{FigLatvian}
\end{figure}

The last Baltic language to classify was Estonian. This is a Finnic language from the Baltic region and thus, in the first run, we used five reference clusters: Slavic II, Romance, Germanic, Baltic and Finnish. However, the similarity with the Romance cluster was low, so we re-ran the classification just for the four remaining clusters. The results are shown in Figure \ref{FigEstonian}. The NMWD histogram confirms a large number of matches with the reference words and the other pictures show a high and almost equal similarity with the Finnish and the Germanic cluster. In this experiment, the Finnish and the Latvian have themselves similarities with the other reference languages, which can be seen from their placement in the plot -- they are a little further from their hypothetical languages than the other languages included in the experiment. 

\begin{figure}[h]
	\centering
   	\parbox{\textwidth}{
   		\parbox{.6\textwidth}{%
      		\centering
      		\includegraphics[width=0.4\textwidth]{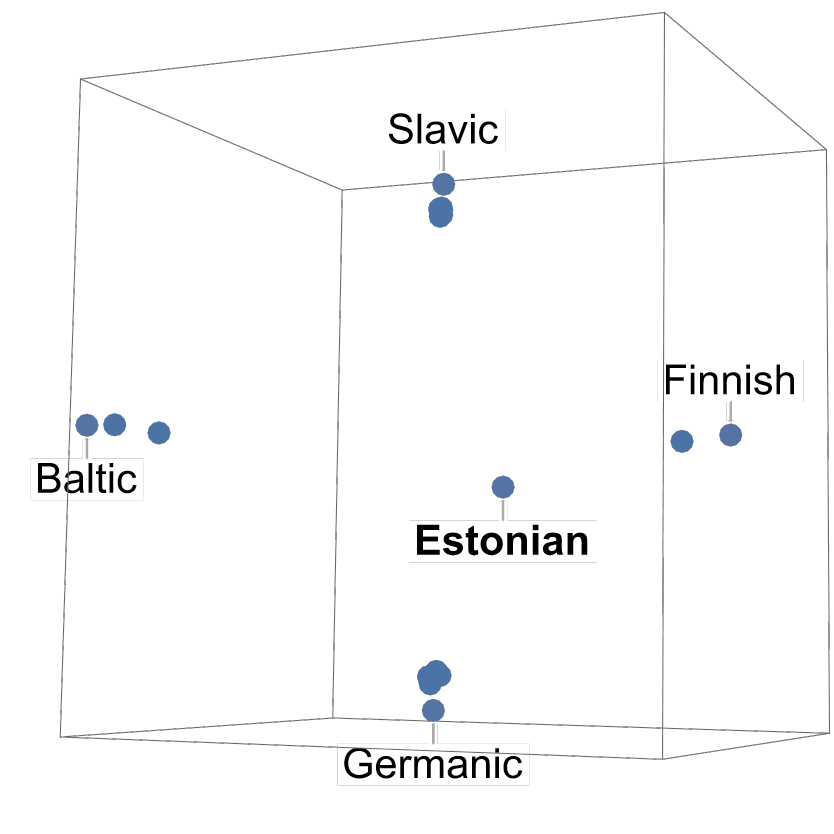}
    	}
   	\parbox{.3\textwidth}{%
      		\includegraphics[width=0.4\textwidth]{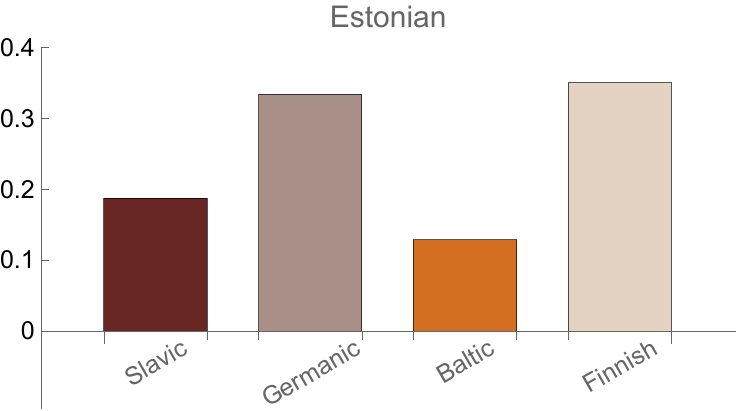}
      		\vskip1em
      		\includegraphics[width=0.4\textwidth]{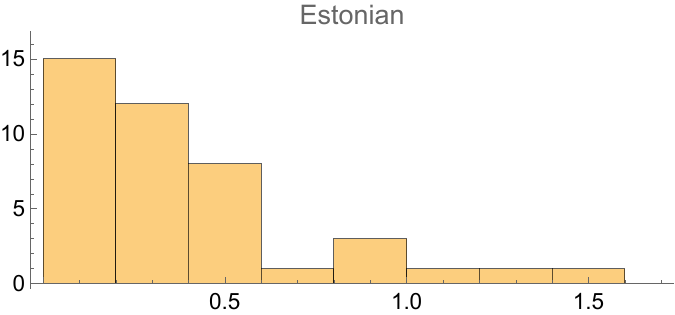}  
    	}
    	}
    	\caption{Classification of Estonian -- the 3D PCA projection of the feature space coordinates, the similarity distribution and the NMWD histogram.}
    	\label{FigEstonian}
\end{figure}

\subsubsection{Maltese}

Malta is a tiny but linguistically interesting country. It is the home of the Europe's only Semitic language, which is Maltese. Moreover, its second official language is English and it is situated in the close proximity of Italy, which brings along possible Germanic and Romance influences. Therefore, we performed an experiment with three reference clusters: Romance, Germanic and Arabic. The results are visualized in Figure \ref{FigMaltese}. This time, we observed a high similarity with the Romance cluster, a little lower similarity with Arabic and only a low similarity with the Germanic cluster. 

\begin{figure}[h]
	\centering
   	\parbox{\textwidth}{
   		\parbox{.6\textwidth}{%
      		\centering
      		\includegraphics[width=0.4\textwidth]{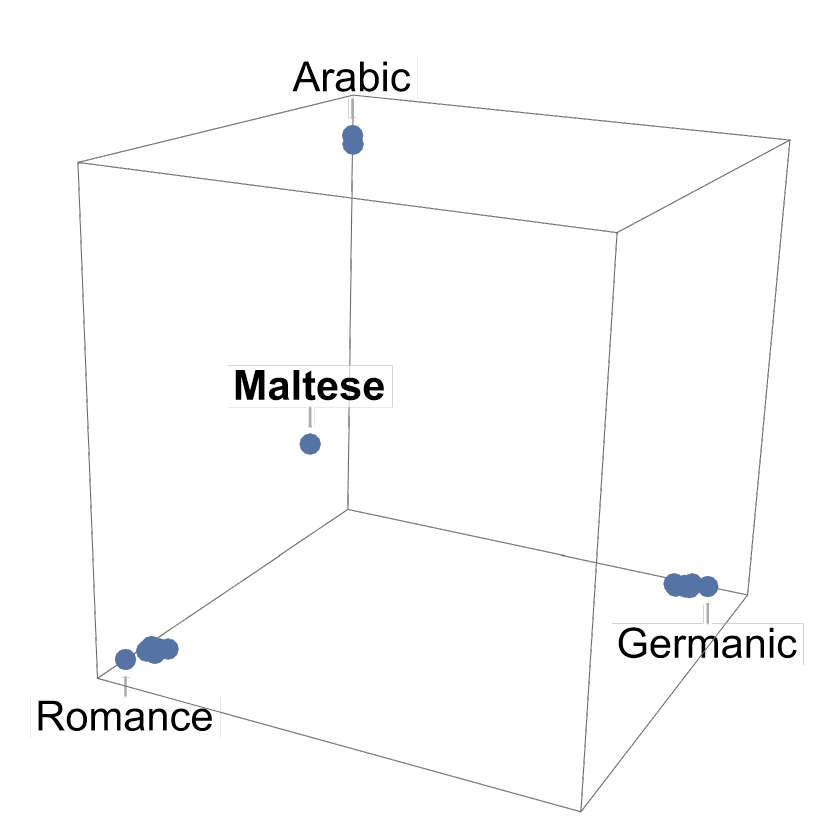}
    	}
   	\parbox{.3\textwidth}{%
      		\includegraphics[width=0.4\textwidth]{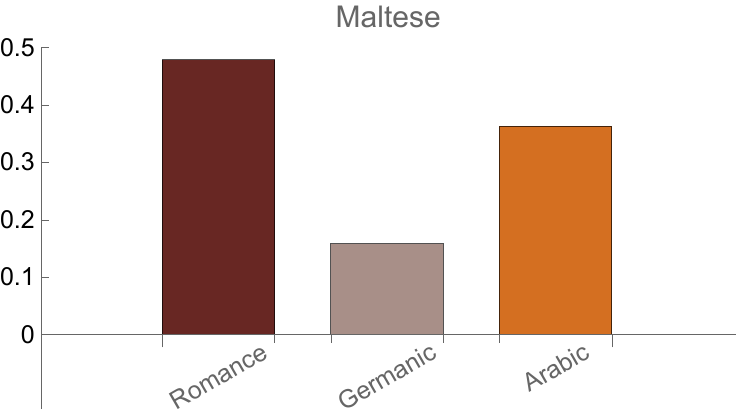}
      		\vskip1em
      		\includegraphics[width=0.4\textwidth]{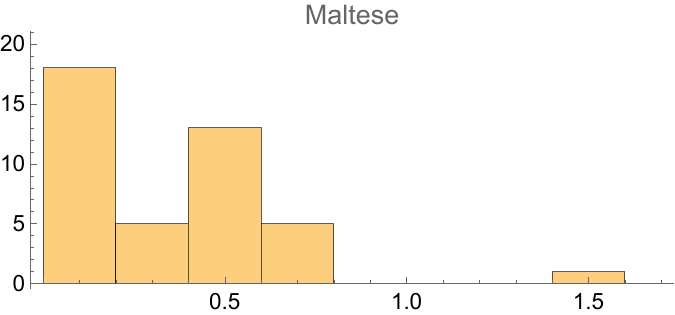}  
    	}
    	}
    	\caption{Classification of Maltese -- the average feature space coordinates, the similarity distribution and the NMWD histogram.}
    	\label{FigMaltese}
\end{figure}

\subsubsection{Breton}

For the last experiment, we moved to the North-West of Europe and tried to classify the Breton language. This is the only Celtic language that survived on the European mainland and it is spoken in a Romance territory, in France. Moreover, Bretagne is close to the British Islands that is the home of all other contemporary Celtic languages. Therefore, in our experiment, we used three reference clusters: Romance, Germanic and Celtic. The results are presented in Figure \ref{FigBreton}. As we can see, the algorithm detected a high similarity with other Celtic languages (Cornish and Welsh) but also a significant similarity with the Romance cluster. 

\begin{figure}[h]
	\centering
   	\parbox{\textwidth}{
   		\parbox{.6\textwidth}{%
      		\centering
      		\includegraphics[width=0.4\textwidth]{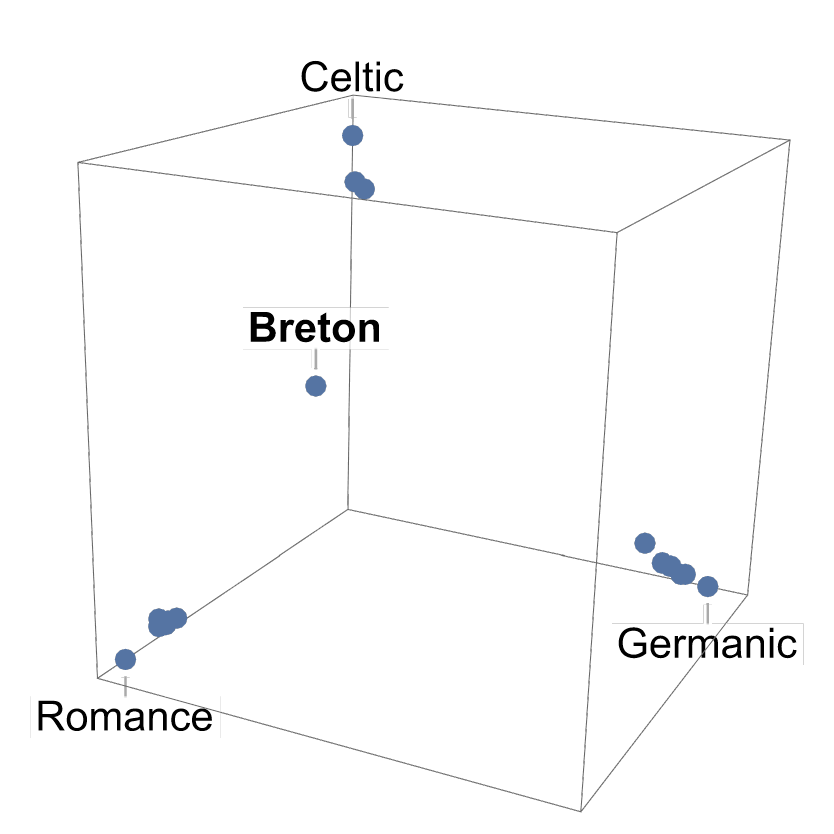}
    	}
   	\parbox{.3\textwidth}{%
      		\includegraphics[width=0.4\textwidth]{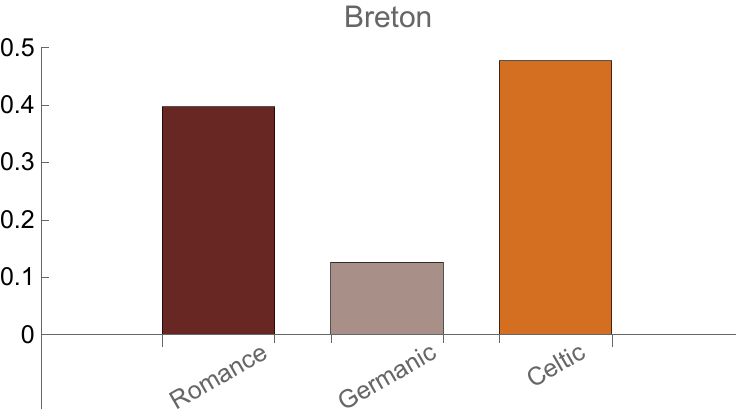}
      		\vskip1em
      		\includegraphics[width=0.4\textwidth]{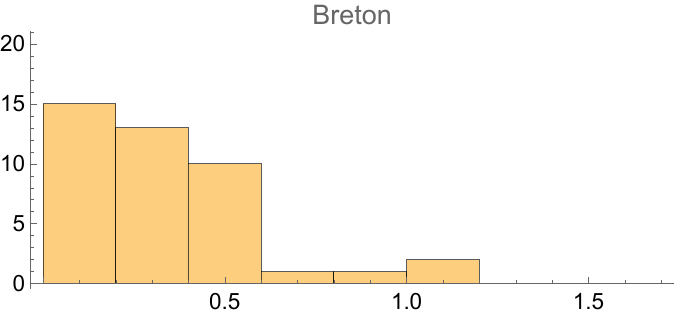}  
    	}
    	}
    	\caption{Classification of Breton -- the average feature space coordinates, the similarity distribution and the NMWD histogram.}
    	\label{FigBreton}
\end{figure}

\section*{Acknowledgements}

This work was supported by the grants APVV-23-0186 and VEGA 1/0249/24.

\end{document}